\documentclass[10pt,reqno]{amsart}
\usepackage{graphicx}
\usepackage{indentfirst,csquotes}
\usepackage{thmtools} 
\usepackage{thm-restate}
\usepackage{natbib}
\usepackage[utf8]{inputenc} 
\usepackage[T1]{fontenc}    
\usepackage{xurl}           
\usepackage{booktabs}       
\usepackage{microtype}      

\usepackage[svgnames]{xcolor}
\definecolor{DarkGreen}{rgb}{0.1,0.5,0.1}
\definecolor{DarkRed}{rgb}{0.5,0.1,0.1}
\definecolor{DarkBlue}{rgb}{0.1,0.1,0.5}

\usepackage{hyperref}
\hypersetup{
    colorlinks=true,       
    linkcolor=DarkBlue,          
    citecolor=DarkGreen,        
    filecolor=DarkRed,    
    urlcolor=DarkBlue,          
}

\usepackage{xspace}
\usepackage{mleftright}
\usepackage{cancel}
\usepackage{bm, bbm}

\usepackage{enumitem}
\setlist{leftmargin=.7cm}

\usepackage{todonotes}

\usepackage{amsthm, amsfonts, amsmath, amssymb, amsthm, dsfont, stmaryrd, mathtools, nicefrac, thm-restate}

\DeclareMathOperator*{\argmax}{argmax}
\DeclareMathOperator*{\argmin}{argmin}

\newcommand{\br}[1]{\mleft(#1\mright)}
\newcommand{\brs}[1]{\mleft[#1\mright]}
\newcommand{\brc}[1]{\mleft\{#1\mright\}}

\let\hat\widehat
\let\tilde\widetilde
\let\bar\overline

\newcommand{\cA}{\mathcal{A}}

\newcommand{\cE}{\mathcal{E}}

\newcommand{\cG}{\mathcal{G}}
\newcommand{\cH}{\mathcal{H}}

\newcommand{\cO}{\mathcal{O}}

\newcommand{\cV}{\mathcal{V}}

\newcommand{\rob}{\textnormal{rob}}
\newcommand{\opt}{\textnormal{opt}}
\newcommand{\base}{\textnormal{base}}
\newcommand{\rad}{\textnormal{rad}}
\newcommand{\suf}{\textnormal{suf}}
\newcommand{\dir}{\textnormal{dir}}
\newcommand{\Conv}{\textnormal{Conv}}

\newcommand{\el}{\textnormal{el}}
\newcommand{\adm}{\textnormal{adm}}
\newcommand{\val}{\textnormal{val}}

\newcommand{\KL}{\textnormal{KL}}

\newcommand{\R} {\ensuremath{\mathbb{R}}} 

\renewcommand{\P}[1]{\mathbb{P}\mleft(#1\mright)}

\newcommand{\ind}[1]{\mathds{1}{\brc{#1}}}

\newcommand{\initOneLiners}{%
 	\setlength{\itemsep}{0pt}
	\setlength{\parsep }{0pt}
  	\setlength{\topsep }{0pt}     	
}

\newcommand{\lE}{\mathbb{E}}

\usepackage{graphicx}
\usepackage{subcaption}

\usepackage{tikz}
\usepackage{pgfplots}
\usetikzlibrary{calc}
\pgfplotsset{compat=1.18}

\usepackage[ruled]{algorithm2e}

\usepackage[nameinlink,noabbrev]{cleveref}

\newtheorem{theorem}{Theorem}

\newtheorem{lemma}{Lemma}
\newtheorem{corollary}[theorem]{Corollary}
\newtheorem{assumption}{Assumption}

\usepackage{crossreftools}
\usepackage{amssymb,amsthm,amsmath}
\usepackage{xcolor, hyperref,titlesec,etoolbox}

\titleformat{\section}{\normalfont\large\bfseries}{\thesection.}{1em}{}
\titlespacing*{\section}{0pt}{2ex}{1ex}
\titleformat{\subsection}{\normalfont\normalsize\bfseries}{\thesubsection.}{0.5em}{}
\titlespacing*{\subsection}{0pt}{1.5ex}{0.5ex}
\titleformat{\subsubsection}{\normalfont\normalsize\itshape}{\thesubsubsection.}{0.5em}{}
\titlespacing*{\subsubsection}{0pt}{1ex}{0.5ex}

\newenvironment{sketchproof}{\noindent\textit{Sketch of Proof.} }{\par\medskip}

\hypersetup{ colorlinks=true, linkcolor=black, filecolor=black, urlcolor=black }
\def\proof{\noindent {\it Proof. $\, $}}
\def\endproof{\hfill $\Box$ \vskip 5 pt }

\usepackage{lipsum}

\begin{document}
\title[]{Do Not Trust the Auctioneer: Learning to Bid in Feedback-Manipulated Auctions} 
\author[]{Luigi Foscari$^\ast$, Matilde Tullii$^\dagger$, Vianney Perchet$^\ddagger$}


\let\thefootnote\relax

\newcommand{\thefootnote}{\fnsymbol{footnote}}

\setcounter{footnote}{1}
\footnotetext{LAILA, Università degli Studi di Milano, \href{mailto: luigi.foscari@unimi.it}{ luigi.foscari@unimi.it}}

\setcounter{footnote}{2}
\footnotetext{Fairplay Team, Crest-Ensae, IP Paris, \href{mailto:matilde.tullii@ensae.fr}{matilde.tullii@ensae.fr}}

\setcounter{footnote}{3}
\footnotetext{Fairplay Team, Crest-Ensae, IP Paris, CRITEO AI Team, Paris \href{mailto:vianney.perchet@ensae.fr}{vianney.perchet@ensae.fr}}

\begin{abstract}
Shilling is the use of artificial bids to make competition appear stronger and push prices upward. We study repeated first-price auctions in which shilling affects feedback but not allocation: the learner wins or loses against the real competing bid, but after a loss observes the maximum of the real bid and an independent shill bid. Thus the manipulation changes what the learner observes and hence how it learns to bid, without changing the outcome of the current auction.
We analyze regret with respect to the best bid benchmark, assuming that the shill-bid distribution is known. Even then, shilling can mask the real bid, while useful side information appears only through intermittent low-shill events. Our algorithm combines a robust interval-elimination branch, which ignores the shilled report and achieves the dynamic-pricing rate $\widetilde{\mathcal O}(T^{2/3})$, with an optimistic branch that debiases losing-side reports and exploits the resulting suffix information when it is reliable and achieves the first-price auctions rate $\widetilde{\mathcal O}(\sqrt{T})$. A validation and racing procedure lets the algorithm use these optimistic updates without knowing the right scale or feedback geometry in advance.
We complement the upper bounds with a matching lower bound, up to logarithmic factors, in the single-active-region case.
Overall, the results show that even feedback-only shilling can sharply alter the statistical difficulty of repeated bidding. 
\end{abstract} 
\maketitle

\section{Introduction}
\label{sec:introduction}

First-price auctions are a standard mechanism for allocating scarce goods, from classical markets to modern online platforms. In large digital markets, bidders repeatedly interact with the same platform, observe outcomes and adapt their bidding rules from the feedback they receive. This makes the feedback channel part of the learning problem: a bidder is not only learning the distribution of competition, but learning it through reports generated by the auction platform.
This creates room for feedback manipulation where a seller or platform need not change the allocation rule to affect future bids; instead it can change what bidders learn. We focus on shilling, where artificial bids or artificial bid reports make competition appear stronger than it is. Such concerns are concrete in online auction markets: shilling is prohibited and monitored on major platforms,\footnote{\url{https://www.ebay.com/help/selling-policies/selling-practices-policy/shill-bidding-policy/policies?id=4353}} while first-price auctions now operate at large scale in markets such as programmatic advertising.\footnote{\url{https://blog.google/products/admanager/update-first-price-auctions-google-ad-manager}} The key point is that manipulation can act through the information channel, rather than through the current allocation.

We study repeated first-price auctions with \emph{max-shilling feedback}. The learner wins or loses against the highest real competing bid. After a loss, however, it observes the maximum of the real competing bid and an independent shill bid. The current winner is unchanged, but the feedback used for future learning is distorted.
This feedback lies between dynamic pricing and ordinary first-price auction feedback. A losing report can sometimes reveal information about bids larger than the one played, but this side information is intermittent and one-sided. High shill bids mask the real bid, while low shill bids reveal suffix information. Even if the shill-bid distribution is known, the learner still does not know where the buyer distribution makes bids profitable.

Our results quantify the price of this manipulation. A robust interval-elimination branch ignores the shilled report and recovers the dynamic-pricing rate \(\widetilde{\cO}(T^{2/3})\). An optimistic branch uses the known shill distribution to debias losing-side reports and exploit reliable suffix information. The resulting rates interpolate between local pricing feedback and richer first-price feedback, and our single-active-region lower bound shows that this interpolation is tight up to logarithmic factors.


\subsection{Contributions}
\label{sec:contributions}

We study contextual first-price bidding with max-shilling feedback, where the observed losing bid may come from a shill. The feedback is richer than bandit feedback, but far from full information. We consider the case in which \(F_S\) is known to the learner. A central contribution is to identify how to use this extra information without assuming more structure than the algorithm needs.

\begin{enumerate}
    \item We give a robust interval-elimination procedure with regret $\widetilde \cO(T^{2/3})$. This guarantee holds even when the shilling feedback is not useful and serves as the safety net for the full algorithm.

    \item We develop an optimistic procedure that extracts additional information from max-shilling feedback. The key technical point is that losing-side observations reveal structured information over active bid intervals, but only through a noisy path-like signal. Turning this into a usable confidence bound requires a geometric uncertainty analysis, which yields faster rates when the active regions are regular and well covered.
    
    \item We make the algorithm adaptive to the unknown shill scale through a dyadic validation scheme. The learner does not need to know the effective shill level in advance. The algorithm validates admissible optimistic candidates at each epoch and keeps the robust branch as backup.

    \item We prove a matching lower bound for the one-active-region case, showing that the rate $\min\{T^{2/3},\sqrt T\,\gamma^{-1/4}\}$ is unavoidable up to logarithmic factors.
\end{enumerate}

\subsection{Related work}
\label{sec:related-works}

Repeated first-price auctions have been studied as online learning problems with censored feedback \citep{han2020optimal,achddouFastRateLearning2021,zhang_leveraging_2022}. 
In the standard stochastic setting, this feedback can be exploited to obtain optimal \(\cO(\sqrt T)\) regret. 
Our model keeps the same allocation and payment rule, but changes the losing-side feedback through shilling, thus the learner must handle both censoring and a filtered information channel.

Our robust benchmark is dynamic pricing, where posted prices reveal only local demand information \citep{besbesDynamicPricingKnowing2009,cesa-bianchiDynamicPricingFinitely2019}. 
This local-feedback structure gives the classical \(T^{2/3}\)-type rate. 
When shilling masks losing-side reports, our model reduces to the same regime. 
When low-shill events are frequent, however, the learner obtains side information about larger bids and can improve beyond \(\cO(T^{2/3})\).

The useful part of our feedback is related to contextual bandits with cross-learning and to feedback-graph models \citep{balseiroContextualBanditsCrosslearning2021,alonBanditsExpertsTale2013,alon_online_2015,erezBestofallworldsOnlineLearning2021}. 
In contextual bandits, complete cross-learning can improve the rate from \(\widetilde{\cO}(\sqrt{CKT})\) to \(\widetilde{\cO}(\sqrt{KT})\). 
In our setting, cross-learning is endogenous: it is generated by the auction outcome, the shill bid, and the unknown buyer distribution. 
This makes the side information one-sided and intermittent.

Our model also relates to bandits with corrupted feedback \citep{lykouris_stochastic_2018,gupta_better_2019,yang_adversarial_2020,lu_stochastic_2021}. 
There, regret usually degrades with an external corruption budget. 
Here the corruption is mechanism-driven: rewards and allocations remain truthful, while only the side information is filtered. 
The relevant parameter is therefore the probability \(\gamma\) of informative low-shill events, rather than a generic corruption level.

Closest to our motivation are works on shilling and credible auction design. 
\citet{akbarpour_credible_2020} study credibility constraints in auctions, while \citet{komo_shill-proof_2024} characterize auction formats that remove the seller's incentive to shill. 
Related work on non-credible auctions studies learning when the seller manipulates payments \citep{wang_learning_2023}. 
Our focus is complementary: we fix the first-price mechanism and study how feedback-only shilling changes the learning rate.

\subsection{Technical challenges}
\label{sec:challenges}

The technical core of the paper is to separate the information that is always reliable from the information that is potentially useful but filtered by shilling. The reported losing bid can reveal more than a single win/loss bit (which is the case for dynamic pricing): when the shill is low enough, one action can teach the learner about a whole suffix of larger bids. This is still far from full information, because the signal is noisy, filtered through the shill process, and useful only when the shill does not mask the buyer. The first challenge is to quantify this intermediate information.

The second challenge is geometric. A losing-side observation does not give equally accurate estimates for all surviving bids. Instead, it behaves like noisy difference information along a path. This is why pointwise UCB recovers only the robust $T^{2/3}$ rate, while faster rates require confidence bounds that combine local observations with suffix information.

The third challenge is adaptivity. The right discretization scale is not known in advance: a grid that is too coarse loses value, while a grid that is too fine wastes samples. The algorithm must validate optimistic updates from data while keeping a robust fallback.

Finally, the lower bound is subtle because the learner gets more than ordinary pricing feedback. A hard instance cannot just hide the best bid and argue that the learner must sample it directly. It must also account for information leaked through losing-side observations, which yields the interpolation between $T^{2/3}$ and $\sqrt T$-type behavior.

\paragraph{Notation}

For any distribution \(D\), let \(F_D\) denote its cumulative distribution function, \(f_D\) its density when it exists, and \(r_D=f_D/F_D\) its reverse hazard rate. For any function \(g\), let \(g'\) denote its derivative and \(g^{-1}\) its inverse when it exists.

\section{Model Setting}
\label{sec:model}




We consider a stochastic repeated first-price auction over \(T\) rounds. At round \(t\), the learner observes a value \(v_t\in[0,1]\), chooses a bid \(p_t\in[0,1]\) and competes against the highest real competing bid \(b_t\). The learner wins if \(p_t\ge b_t\), receives reward \(v_t-p_t\), and otherwise receives reward \(0\) and observes the shilled report \(o_t=\max\{b_t,s_t\}\), where \(s_t\) is a shill bid. The protocol is described in \cref{fpa-shilling}.

\begin{algorithm}[t]
    \DontPrintSemicolon
    \For{time $t = 1, \dots, T$}{
        A value $v_t \in [0, 1]$ is revealed to the agent\;
        The other bidders' private maximal bid is $b_t \in [0, 1]$\;
        The auctioneer's private \emph{shill} bid is $s_t \in [0, 1]$\;
        The agent bids a price $p_t \in [0, 1]$\;
        \lIf{the auction is won $(p_t \ge b_t)$}{
            The agent receives reward $v_t - p_t$
        }\lElse{
            The agent observes the \emph{shilled} maximal bid $o_t = \max\{b_t, s_t\}$
        }
    }
    \caption{First-price auctions with max-shilling}
    \label{fpa-shilling}
\end{algorithm}

We assume that \(v_t\sim V\), \(b_t\sim B\), and \(s_t\sim S\) are drawn independently across rounds, with all three distributions supported on \([0,1]\), and that \(s_t\) is independent of \(b_t\). We call \(O=\max\{B,S\}\) the reported losing-side distribution induced by shilling.
For any value \(v\), the expected utility of bid \(p\) is
\begin{equation}
    \label{eq:utility}
    u_B(v,p)\coloneq \mathbb E[(v-p)\ind{p\ge B}]
    =(v-p)F_B(p).
\end{equation}
We measure regret against the best bid for the realized value
\[
    R_T \coloneq
    \mathbb E\left[
    \sum_{t=1}^T
    \left(
        \max_{p\in[0,1]} u_B(v_t,p)-u_B(v_t,p_t)
    \right)
    \right],
\]
where the expectation is taken with respect to the random draw of $b_t$, $v_t$, $s_t$ and the (possible) internal randomization of the agent's strategy.

The shill bid affects only the losing-side report. When \(S\equiv 0\), the learner observes the truthful highest competing bid after a loss, recovering the standard repeated first-price feedback model. For general \(S\), the report can mask the real competing bid, and our goal is to exploit the shilled report when it is informative while retaining the dynamic-pricing fallback when it is not.


\subsection{Impact of shilling}
\label{sec:why-max-shilling-is-hard}

Max-shilling is not meant to be the only possible model of shilling, since an auctioneer could manipulate reports in many other ways, possibly depending on the learner's bid, the realized competing bid or the past history of the interaction. We focus on this model because it is simple and transparent while already capturing the main difficulty: even when the manipulation only raises the losing-side report through an independent shill bid, it can bias the learner toward higher bids and partially destroy the cross-learning structure that makes ordinary first-price auction feedback statistically useful.
The max-shilling rule provides a minimal model of feedback manipulation: although it biases reported competition upward, it preserves partial information about the true bid distribution through some losing reports. As a result, the learning problem interpolates between first-price learning and dynamic pricing.

Since \(B\) and \(S\) are independent, the reported bid distribution satisfies \(F_O(p)=\Pr(B\le p,S\le p)=F_B(p)F_S(p)\), and therefore \(F_O(p)\le F_B(p)\) for every \(p\). An unaware learner that treats \(O\) as the true competing bid estimates a distribution in which low competing bids appear less often, so the market looks more competitive than it really is.

This upward distortion is aligned with the auctioneer's objective. Fix a value \(v\in[0,1]\), and write \(u_B(p)\coloneq (v-p)F_B(p)\) and \(u_O(p)\coloneq (v-p)F_O(p)\). At an interior optimum of \(u_B\), whenever \(F_B\) is differentiable, the first-order condition is \(u_B'(p)=0\), equivalently \((v-p)r_B(p)=1\), where \(r_B(p)=f_B(p)/F_B(p)\) is the reverse hazard rate. Under max-shilling, the reverse hazard rate of the reported bid distribution decomposes as follows.

\begin{restatable}{lemma}{rhrdecomposition}
\label{lem:rhr-decomposition}
For all \(p\in(0,1)\) where the reverse hazard rates are well defined, \(r_O(p)=r_B(p)+r_S(p)\).
\end{restatable}

The proof is given in \cref{app:proof-of-rhr-decomposition}. The identity says that, locally, optimizing against the reported distribution increases the reverse hazard rate appearing in the first-order condition, which is the usual force pushing first-price optimal bids upward.

The next statement gives a global version of this intuition and does not require a unique optimum. For \(D\in\{B,O\}\), define \(A_D(v)\coloneq \argmax_{p\in[0,1]}(v-p)F_D(p)\) and \(p_D^{\max}(v)\coloneq \max A_D(v)\).

\begin{restatable}{lemma}{optimaundershilling}
\label{lem:optima-under-shilling}
Let \(S\) be any distribution supported on \([0,1]\). Then, for every value \(v\in[0,1]\),
\begin{enumerate}
    \item The rightmost optimal bid never decreases: \(p_O^{\max}(v)\ge p_B^{\max}(v)\).
    \item New optimal bids cannot be created to the left of \(p_B^{\max}(v)\):
    \[
        \text{for all } p\in[0,p_B^{\max}(v)] \qquad p\in A_O(v) \implies p\in A_B(v) \text{ and } F_S(p)=F_S(p_B^{\max}(v)).
    \]
\end{enumerate}
\end{restatable}

The proof is given in \cref{app:proof-of-optima-under-shilling}. The reason is that \(u_O(p)=u_B(p)F_S(p)\), and for every \(p\le p_B^{\max}(v)\), both factors are no larger than their values at \(p_B^{\max}(v)\). Thus no smaller bid can strictly beat \(p_B^{\max}(v)\) under the shilled objective.

If the honest optimum is unique and the shill density is positive there, then the upward shift is strict.

\begin{restatable}{lemma}{singleoptimaundershilling}
\label{lem:single-optima-under-shilling}
Assume that \(u_B\) has a unique interior maximizer \(p_B^\star(v)\in(0,1)\). If \(F_S\) is differentiable at \(p_B^\star(v)\) and \(f_S(p_B^\star(v))>0\), then \(p_O^{\max}(v)>p_B^\star(v)\).
\end{restatable}

The proof is given in \cref{app:proof-of-single-optima-under-shilling}. The intuition is that \(u_B'(p_B^\star(v))=0\), while \(u_O'(p_B^\star(v))=u_B(p_B^\star(v))f_S(p_B^\star(v))>0\), so the reported objective is still increasing at the honest optimum.

This also gives a simple offline-learning interpretation. Suppose a learner forms its bidding rule from historical reports, some honest and some max-shilled, and the origin of each report is independent of the learner's future bids. Then the empirical CDF converges to 
\[
    F_\lambda(p)=\lambda F_B(p)+(1-\lambda)F_O(p) \qquad \text{for any } \lambda\in[0,1].
\]
Since \(F_O(p)=F_B(p)F_S(p)\), this can be written as \(F_\lambda(p)=F_B(p)G_\lambda(p)\), where \(G_\lambda(p)=\lambda+(1-\lambda)F_S(p)\) is nondecreasing. The induced objective is \(u_\lambda(p)=(v-p)F_\lambda(p)=u_B(p)G_\lambda(p)\), so the mixture behaves like the honest objective multiplied by an increasing weight.

\begin{restatable}{lemma}{mixtureoptimum}
\label{lem:mixture-optimum}
For every \(\lambda\in[0,1]\), define \(A_\lambda(v)\coloneq \argmax_{p\in[0,1]} u_\lambda(p)\) and \(p_\lambda^{\max}(v)\coloneq \max A_\lambda(v)\). Then, for every value \(v\in[0,1]\), \(p_\lambda^{\max}(v)\ge p_B^{\max}(v)\). Moreover, if \(0\le\lambda_1\le\lambda_2\le1\), then \(p_{\lambda_1}^{\max}(v)\ge p_{\lambda_2}^{\max}(v)\). Consequently, \(p_B^{\max}(v)=p_1^{\max}(v)\le p_\lambda^{\max}(v)\le p_0^{\max}(v)=p_O^{\max}(v)\).
\end{restatable}

The proof is given in \cref{app:proof-of-mixture-optimum}. The key observation is that, when \(\lambda_1\le\lambda_2\), the ratio \(G_{\lambda_1}(p)/G_{\lambda_2}(p)\) is nondecreasing in \(p\). Moving from \(\lambda_2\) to \(\lambda_1\) therefore puts relatively more weight on higher bids, so the rightmost maximizer cannot move left.

The previous statement is written for rightmost maximizers, so it does not require single-peaked utilities. In the usual single-peaked case, it has the simpler interpretation that the unique optimizer moves monotonically from the honest optimum toward the fully shilled optimum.

\begin{restatable}{corollary}{mixturesinglepeaked}
\label{cor:mixture-single-peaked}
Assume that, \(\forall \ \lambda\in[0,1]\), the objective \(u_\lambda\) has a unique maximizer \(p_\lambda^\star(v)\). Then, \(\forall \ 0\le\lambda_1\le\lambda_2\le1\), \(p_{\lambda_1}^\star(v)\ge p_{\lambda_2}^\star(v)\). In particular, \(p_B^\star(v)=p_1^\star(v)\le p_\lambda^\star(v)\le p_0^\star(v)=p_O^\star(v)\).
\end{restatable}

Thus, even passive contamination by max-shilled reports can bias learning toward the shilled benchmark. The online problem is more challenging because the learner’s own bids determine which censored observations are revealed. Still, max-shilling does not destroy all side information, for every \(q\ge p\), \(\Pr(p<B\le q,\ O\le q)=F_S(q)(F_B(q)-F_B(p))\), thus \(F_S(q)\) is the fraction of useful losing-side information that survives at level \(q\). If \(F_S(q)\) is small, max-shilling hides most of the side information and the learner is close to the dynamic-pricing regime; while if \(F_S(q)\) is bounded away from zero, enough suffix information survives to recover faster first-price-auction rates.

\section{Upper bound}
\label{sec:upper-bound}

\subsection{The shape of the utility curve}

Shape restrictions on the revenue or utility curve are standard in dynamic pricing and first-price auction learning, where they are used to make price optimization well behaved and to obtain local curvature near the optimum \citep{javanmard2019dynamic,fan2024policy,achddouFastRateLearning2021}. We make a similar but weaker assumption tailored to interval elimination: rather than requiring uniqueness of the optimal bid, we only require enough structure for the upper level sets of the honest utility to be intervals.
\begin{assumption}[Weak RHR condition]
\label{ass:weak-rhr-buyer}
The buyer distribution \(B\) has no atom at \(0\), is continuous on \([0,1]\), is differentiable on \((0,1)\), and satisfies \(F_B(p)>0\) and \(f_B(p)>0\) for every \(p\in(0,1)\). Moreover, the map
$
    \phi_B(p)\coloneq p+\frac{1}{r_B(p)}
$
is nondecreasing on \((0,1)\).
\end{assumption}

\begin{restatable}{lemma}{rhrimpliesintervallevelsets}
\label{lem:rhr-implies-interval-level-sets}
Under \cref{ass:weak-rhr-buyer}, for every value \(v\in[0,1]\) and every level \(\alpha\), the set
$
    \{p\in[0,1]:u_B(v,p)\ge\alpha\}
$
is a possibly empty interval. Consequently, for every bid grid \(0=q_1<\cdots<q_K=1\), the set
$
    \{i:u_B(v,q_i)\ge\alpha\}
$
is an interval of grid indices.
\end{restatable}
The proof of this result can be found in \cref{app:proof-of-rhr-implies-interval-level-sets}.
Thus, after intersecting with a bid grid, the near-optimal bids for each value can be represented by a left and right endpoint. This is the only reason we impose \cref{ass:weak-rhr-buyer}. The statistical identity for known \(F_S\), and the construction of the debiased estimator, hold for arbitrary \(F_B\) and \(F_S\) as long as \(F_S(q)>0\) on the grid points being used.

\subsection{Overview of the algorithm}

The high-level structure of the algorithm (\cref{alg:pseudocode}) is a successive elimination routine on a discretized set of prices and values. The algorithm proceeds in epochs and, at each epoch $m$, maintains an active set \(A_m(v)\), corresponding to the interval containing the optimal arms, for each value $v$ in the grid $\cV$.
Throughout the run, the algorithm builds two estimates of the utility gap
\[
\Delta(v,q,p)=(v-q)F_B(q)-(v-p)F_B(p), \qquad \forall v \in \cV, \forall p, q \in A_m(v)
\]
The two estimates, denoted respectively as \(\widehat\Delta^{\rob}_t(v,q,p)\) and \(\widehat\Delta^{\opt}_t(v,q,p)\), are built using the information that the algorithm receives by either playing a robust subroutine, on odd rounds, or an optimistic one during even time steps.

The robust estimator is constructed only employing the auction outcomes \(\ind{p_t\ge b_t}\). On one side, since it ignores the additional information $o_t$, it is robust to the auctioneer's manipulation; on the other, it can rely on fewer bits of information since it relies only on direct samples, avoiding information sharing across different prices. Formally, this subroutine plays prices in a round robin fashion, building an estimate $\hat{F}_t$ for each point as the empirical average of the auction outcomes. 
The robust gap estimate is then
\begin{align}\label{eq:D_rob}
\widehat\Delta^{\rob}_t(v,q,p)
=
(v-q)\widehat F_t(q)-(v-p)\widehat F_t(p).
\end{align}

The optimistic estimator instead considers the reported losing bid \(o_t\), thus allowing cross-learning among different grid points, after correcting for the known shill distribution. Upon playing \(p_t\), the algorithm collects the following suffix observation
\[
Y_t(q)
=
\ind{p_t\ge b_t}
+
\ind{p_t<b_t}
\frac{\ind{o_t\le q}}{F_S(q)}, \qquad \forall q\leq p_t
\]
Conditional on the price played, this constitutes an unbiased estimate of $F_B(q)$. This allows to recover part of the suffix information available in the classical first-price auction setting\footnote{By playing the smallest value $p_t = \argmin A_{t-1}(v)$, \citet{han2020optimal} manage to recover full-information and achieve the $\sqrt{T}$ rate in classical first-price auctions.} \(q\ge b_t\).

The optimistic estimator is built to use this information when it is reliable and to discount it when it is noisy. Moreover, instead of building independent pointwise estimates of \(F_B(q)\), which would be impacted by the correlation of the information,
it converts suffix observations into measurements of adjacent utility differences along the active interval.
Concretely, if \(A_m(v)\cap \cG_m=\{q_1<\cdots<q_k\}\) and \(a_j=v-q_j\), then the optimistic branch forms suffix-difference measurements
\[
Y^{\suf}_j
=
a_{j+1}Y_t(q_{j+1})-a_jY_t(q_j),
\qquad j=1,\ldots,k-1.
\]
Again these measurements are conditionally unbiased estimates for \(u_B(v,q_{j+1})-u_B(v,q_j)\).

Additionally, these measurements are combined with direct win/loss measurements, given by the auction outcome as in the robust subroutine, in a weighted least-squares certificate. More precisely, direct measurements are entered with unit weight, while suffix ones are scaled by a quantity \(\omega_t\).

The role of \(\omega_t\) is to ensure that direct and suffix measurements have the same statistical scale. Indeed, a direct observation has bounded variance, while a suffix measurement is made noisier by the term \(\nicefrac{1}{F_S(q)}\) appearing in the definition of the estimator. If \(F_S(q)\) is small, then low-shill events are rare and the correction is unstable.
For an active interval of utility width \(w\) and grid spacing \(h\), if \(F_S(q)\) is at least \(\gamma\) and is locally regular on the interval, the suffix covariance is of order $\nicefrac{w+h}{\gamma}$, which suggests fixing \(\omega_t\simeq \nicefrac{\gamma}{w+h}\). In epoch \(m\), the algorithm replaces this unknown scale by dyadic candidates \(\bar\gamma\in\Gamma\), using \(\omega_{m,\bar\gamma}=\nicefrac{\bar\gamma}{C_{\suf}(2^{-m}+h_m)}\) and allowing only candidates that are admissible on the active grid to validate the optimistic certificate.
Thus the algorithm trusts suffix information more when low-shill events are frequent and the active interval is already localized, while the dyadic validation step removes the need to know \(\gamma\) in advance.



Formally, all optimistic observations are represented as linear measurements \(Y_s=\Phi_sF+\xi_s\), where \(F\) is the vector of grid values \(F_B(q)\) and \(\Phi_s\) is the corresponding design row.. Hence, the optimistic branch constructs
\[
G_t=\sum_s\omega_s\Phi_s^\top\Phi_s,
\qquad
b_t=\sum_s\omega_s\Phi_s^\top Y_s,
\qquad
\widehat F_t=G_t^\dagger b_t.
\]
From this, it is possible to define the optimistic gap estimate for a couple of points $q_i, q_j$ as
\begin{align}\label{eq:def_D_opt}
  \widehat\Delta^{\opt}_t(v,i,j)=g_{v,i,j}^\top\widehat F_t\,,   \qquad \text{with} \ g_{v,i,j}=(v-q_i)e_i-(v-q_j)e_j\,.
\end{align}
 
Throughout each epoch, the algorithm maintains two quantities $\hat{r}_{\text{opt}}(t)$ and $\hat{r}_{\text{rob}}(t)$, which estimate the average over values $v$ of the largest confidence interval around the gap estimates $\widehat{\Delta}^{\text{opt}}$ and $\widehat{\Delta}^{\text{rob}}$ respectively. Epoch $m$ terminates as soon as one of these quantities is confidently below the threshold $2^{-(m+1)}$. At the end of the epoch, an elimination step is performed: any arm for which $\widehat{\Delta}^b(v, p, q)$ is confidently positive is removed from the active set, where $\widehat{\Delta}^b$ denotes either $\widehat{\Delta}^{\text{opt}}$ or $\widehat{\Delta}^{\text{rob}}$, depending on which of $\hat{r}_{\text{opt}}$ or $\hat{r}_{\text{rob}}$ first fell below the threshold $2^{-(m+1)}$.

Lastly, observation relies on some discretization grids for the price and the values.
Choosing a partition for the set of values is straightforward, since the utility function is linear in this component, hence a grid such that $\vert \cV\vert=T$ will ensure a value discretization error of $\cO(1)$. 
Conversely, choosing a price discretization scheme introduces an additional layer of complexity to the problem, as the algorithm must bridge two distinct regret regimes, namely $\sqrt{T}$ and $T^{2 / 3}$. A fixed grid cannot achieve optimal performance uniformly across instances, as it would necessarily penalize one of the two regimes. To overcome this issue, the algorithm relies on an adaptive discretization scheme that couples the approximation error to the statistical hardness of the problem. As a result, the cumulative discretization error scales at the same rate as the underlying learning regret.



\begin{algorithm}
    \DontPrintSemicolon
    \For{$\text{epochs}=0, 1, \ldots$}{
    Select grid $\cG_m$\\
    \While{$\hat{r}^{\opt}(t)$ and $ \hat{r}^{\rob}(t)>2^{-(m+1)}$}{
    \uIf(\tcp*[f]{Robust branch}){$t$ odd}{
        Update $\hat{\Delta}^{\rob}(v, p, q)$ as \cref{eq:D_rob} and $\hat{r}^{\rob}$
    }\ElseIf(\tcp*[f]{Optimistic branch}){$t$ even}{
        Try all dyadic shill scales \(\bar\gamma\in\Gamma\) and use the largest validating candidate\;
        Update $\hat{\Delta}^{\opt}(v, p, q)$ as in \cref{eq:def_D_opt} and $\hat{r}^{\opt}$
    } 
    }
    Let $b \in \{\rob,\ \opt\}$ such that $\hat{r}^{\,b}\leq 2^{-(m+1)}$\;
    Eliminate all the prices $p$ for which exists $q$ such that $\widehat \Delta^{b}(v, p, q) >2^{-(m+1)}$ \;
    }
\caption{Pseudocode for \Cref{alg:known_FS}}\label{alg:pseudocode}
\end{algorithm}

\paragraph{Regularity of the shill distribution.}
The optimistic branch uses losing-side reports only through certified active intervals. On such intervals, we need a mild local regularity condition on the known shill distribution \(F_S\).
\begin{assumption}[Local regularity of the shill distribution]
\label{ass:local-shill-regularity}
Let \(I=\{q_{i_1}<\cdots<q_{i_k}\}\subseteq G\) be an interval of grid points with mesh \(h\). We say that \(F_S\) is \((\gamma,C_S)\)-regular on \(I\) if the following two conditions hold:
\[
F_S(q_{i_j})\ge \gamma
\qquad\text{for all } q_{i_j}\in I,
\]
and
\[
F_S(q_{i_{j+1}})-F_S(q_{i_j})\le C_S\gamma h
\qquad\text{for all adjacent } q_{i_j},q_{i_{j+1}}\in I.
\]
\end{assumption}

The first condition ensures that low-shill events occur with probability at least \(\gamma\) throughout the interval, so the debiased suffix observations are not too unstable. The second condition rules out sharp local jumps of \(F_S\) inside the same active interval.
This assumption is only required for intervals on which the optimistic branch attempts to certify an update. If \(F_S\) is not \((\gamma,C_S)\)-regular on the intervals considered by the optimistic branch, the optimistic certificate is not accepted and the algorithm falls back to the robust branch.



The rate of the optimistic branch depends on how many separated parts of the bid space must be certified at the same time. At epoch \(m\), the algorithm has one active interval \(A_m(v)\) for each value grid point \(v\in V\). We say that the instance has at most \(\rho\) active bid regions if, at every epoch \(m\), the union \(\bigcup_{v\in V} A_m(v)\) can be covered by at most \(\rho\) disjoint intervals of grid points. The case \(\rho=1\) means that, although different values may have different active intervals, these intervals all lie in one connected region of the bid grid. The proof of the next result can be found in \Cref{sec:proof_upper}.

\begin{theorem}[Upper bound]
\label{thm:upper-bound}
Assume \cref{ass:weak-rhr-buyer} and suppose that \(F_S\) is known to the learner. If, at every epoch, the active intervals maintained by the optimistic branch can be covered by at most \(\rho\) active bid regions, and \(F_S\) is \((\gamma,C_S)\)-regular on each such region in the sense of \cref{ass:local-shill-regularity}, choosing a price grid $\cV$ such that $\vert\cV\vert = T$ and confidence parameter $\delta = 1 / T$, then the algorithm satisfies
\[
R_T
\le
\widetilde O\!\left(
\min\left\{
T^{2/3},
\rho^{1/4}\sqrt T\,\gamma^{-1/4}
\right\}
\right).
\]
In particular, when the active intervals form a single active bid region, \(\rho=1\), the bound becomes
\[
R_T
\le
\widetilde O\!\left(
\min\left\{
T^{2/3},
\sqrt T\,\gamma^{-1/4}
\right\}
\right).
\]
If the regularity condition fails on the intervals considered by the optimistic branch, or if the optimistic certificate does not validate, the same algorithm falls back to the robust branch and still guarantees \(R_T\le \widetilde O(T^{2/3})\).
\end{theorem}

\section{Lower Bound}
\label{sec:lower-bound}

We now prove a lower bound for the one-active-region case, corresponding to $\rho=1$ in the packed-region upper bound. The construction follows the hidden-cell strategy used in dynamic pricing and continuum-armed bandits, where one hides a small favorable region in a continuous action space; see, for example, \citet{kleinbergValueKnowingDemand2003, besbesDynamicPricingKnowing2009, kleinbergNearlyTightBounds2004}.
The starting point is a flat hard instance. We construct a buyer distribution for which all bids in some interval have exactly the same expected utility. Before seeing data, the learner has no reason to prefer one bid in this interval over another. We then modify the instance by planting a small utility bump in one unknown cell of the interval. The learner can achieve low regret only if it finds this planted cell, because bids outside it are slightly but consistently suboptimal.

Unlike the mentioned lower bound results, here the reported losing bid can sometimes reveal more global information: even a bid outside the planted cell may leak information about where the buyer distribution was perturbed. Thus the lower bound cannot simply reuse the classical pricing argument.
The proof shows that this extra information helps, but only to a limited extent: if informative shill events are very rare, the learner is essentially forced to search locally and the usual $T^{2/3}$ dynamic-pricing rate appears; if informative shill events are more common, the learner can exploit the losing-side reports and improve on $T^{2/3}$. However, the hidden cell remains hard to identify quickly enough, and the regret is still at least of order $\sqrt T \gamma^{-1/4}$. This gives the one-region lower bound matching our upper bound up to logarithmic factors.

\begin{restatable}[Single-region lower bound]{theorem}{singleregionlowerbound}
\label{th:single-region-lower-bound}
Fix $ T\ge 1$ and $ \gamma\in(0,1]$. There is a constant $ c>0$ such that, for every learning algorithm $ \pi$, there is an oblivious max-shilling instance satisfying \cref{ass:weak-rhr-buyer} and satisfying $F_S(q) = \gamma$ for every $q$ on a hard bid interval $J$, for which
\[
    R_T(\pi) \ge c \min\brc{ T^{2/3}, \sqrt T\,\gamma^{-1/4}}.
\]
\end{restatable}    
\begin{sketchproof} The complete proof of this theorem can be found in \cref{app:proof-of-single-region-lower-bound}.
The construction has two ingredients. First, we choose a hard interval \(J\) on which the baseline utility is flat. More precisely, we take a buyer distribution \(F_0\) such that, for the fixed value \(v=1\), the utility \(u_0(p)=(1-p)F_0(p)\) is constant on \(J\). Thus, before perturbation, all bids in \(J\) are equally good. We then divide \(J\) into \(N\) cells of width \(h\) and choose one hidden cell \(A\) uniformly at random. In instance \(A=a\), we add a small smooth bump of height \(\varepsilon\) to the utility on cell \(a\). The resulting buyer distribution is still a valid CDF as long as \(\varepsilon\lesssim h\), and the utility still has interval upper level sets.

The learner must identify the planted cell to obtain small regret. Let \(P_a^\pi\) be the transcript law of policy \(\pi\) when the bump is in cell \(a\), and let \(P_0^\pi\) be the transcript law under the flat baseline. We compare every planted instance to this common reference law. By the chain rule for KL divergence over the adaptive transcript, the average reference KL satisfies \(\frac1N\sum_{a=1}^N \KL(P_0^\pi\|P_a^\pi)\lesssim T(h\varepsilon^2+\gamma\varepsilon^2/h)\). The first term comes from direct win/loss observations: the perturbation is supported on one cell, so after averaging over the random planted cell, a fixed bid falls in the perturbed region with probability of order \(h\). The second term comes from losing-side observations. These are informative only on the low-shill branch, which has probability \(\gamma\) on \(J\); on that branch, the local density perturbation has scale \(\varepsilon/h\), giving order \(\gamma\varepsilon^2/h\) per round. Hence the average information collected over \(T\) rounds is of order \(T(h\varepsilon^2+\gamma\varepsilon^2/h)\) (see \cref{fig:lower-bound}).

We choose the numerical constant in \(\varepsilon\) so that this average reference KL is bounded by a sufficiently small constant. Now let \(\phi(Z)\in[N]\) be any estimator of the planted cell, and write \(E_a=\{\phi(Z)=a\}\). Since the events \(E_a\) are disjoint and exhaustive under the baseline law, \(N^{-1}\sum_a P_0^\pi(E_a)=1/N\). Pinsker's inequality gives \(P_a^\pi(E_a)\le P_0^\pi(E_a)+\sqrt{\KL(P_0^\pi\|P_a^\pi)/2}\). Averaging over \(a\) and applying Jensen's inequality shows that, if the average reference KL is small enough and \(N\ge 4\), no estimator can recover the planted cell with probability better than a constant. Equivalently, every estimator fails to identify the planted cell with constant probability on average over the planted instances.

To turn this testing statement into regret, define \(\phi_\pi\) from the learner's actions as the cell whose good region is visited most often. If \(\phi_\pi\ne A\), then the learner must have bid outside the true good region on a constant fraction of the rounds. Each such bid loses \(\Omega(\varepsilon)\), so some planted instance has regret at least \(\Omega(T\varepsilon)\).

It remains to tune \(h\) and \(\varepsilon\). If \(\gamma\le T^{-2/3}\), take \(h\asymp\varepsilon\asymp T^{-1/3}\), which gives regret \(\Omega(T^{2/3})\). If \(\gamma>T^{-2/3}\), take \(h\asymp\sqrt\gamma\) and \(\varepsilon\asymp T^{-1/2}\gamma^{-1/4}\), which gives regret \(\Omega(\sqrt T\,\gamma^{-1/4})\). Combining the two regimes gives the stated lower bound. 
\end{sketchproof}

\begin{figure}
    \centering
    \includegraphics[width=.8\textwidth, page=1]{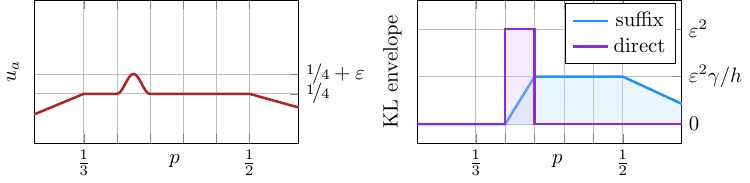}
    \caption{Schematic lower-bound instance and information profile. On the left, the utility $u_a$ is flat on the hard interval $J = [1/3, 1/2]$ apart from one hidden cell (in this case cell number $a = 2$ out of $N = 5$) which contains a bump of height $\varepsilon$. On the right, the order of the per-round KL envelope: direct local testing is based on the outcome of the auction and gives a strong signal of order $\varepsilon^2$, but only on the hidden cell; cross-learning leverages the shilled maximal bid $o_t$ and is of order $\gamma\varepsilon^2/h$, because we can extract information only on the \emph{low-shill} event, which occurs with probability $\gamma$, but let's us recover information spread over a suffix of larger bids.}
    \label{fig:lower-bound}
\end{figure}

The lower bound continues to hold when the learner knows the shill distribution \(F_S\). In the construction, \(F_S\) is the same for all planted alternatives: it places mass \(\gamma\) below the hard interval \(J\) and mass \(1-\gamma\) at \(1\), so \(F_S(q)=\gamma\) for every \(q\in J\). Thus revealing \(F_S\) only tells the learner how often losing-side reports are informative; it does not reveal which planted buyer distribution \(F_a\) is in force. Since the hidden index affects only \(F_B\), the KL comparison to the baseline and the subsequent Pinsker testing argument are unchanged.

\begin{corollary}
When $\rho = 1$, the upper bound gives \( \widetilde \cO (\min\{T^{2/3}, \sqrt T\,\gamma^{-1/4}\}) \).
The lower bound above matches this rate up to logarithmic factors. Hence the algorithm is minimax-optimal, up to logarithms, in the one-active-region case.
\end{corollary}

A more general packed-region upper bound would depend on the number of distinct active bid regions, mimicking the structure in \cref{sec:upper-bound}. It is tempting to try to match it by replicating the hidden-cell construction on several disjoint bid regions. but this is not immediate in the max-shilling model as a bid in a lower active region may, through the informative losing-side report, reveal information about higher bid regions without causing additional regret. Thus off-region information cannot simply be charged to regret. For this reason, the matching lower bound for multiple packed active regions is left as a separate question.

\section{Conclusion}

We study repeated first-price auctions in which shilling affects feedback but not allocation. Our analysis shows why this problem is technically demanding as useful information is present, but it is hidden inside correlated path-like observations whose reliability depends on the shill distribution, the active bid geometry and the current localization scale. The algorithm addresses this by combining interval elimination, weighted least-squares certificates, validation and racing over unknown scales. This should be viewed as a first step rather than a complete theory of feedback manipulation in auctions. Extending these ideas to richer shilling models, unknown shill distributions, broader auction formats and more general active-region geometries remains open, but the results suggest that carefully exploiting the structure of manipulated feedback can recover much more than a purely robust approach would allow.



\section*{Acknowledgements}
LF acknowledges support by the MUR PRIN grant 2022EKNE5K (Learning in Markets and Society), by the FAIR (Future Artificial Intelligence Research) project, funded by the NextGenerationEU program within the PNRR-PE-AI scheme (M4C2, investment 1.3, line on Artificial Intelligence), and by the EU Horizon CL4-2022-HUMAN-02 research and innovation action under grant agreement 101120237, project ELIAS (European Lighthouse of AI for Sustainability). VP’s research was supported in part by the French National Research Agency (ANR) in the framework of the PEPR IA FOUNDRY project (ANR-23-PEIA-0003) and through the grant DOOM ANR-23-CE23-0002. It was also funded by the European Union (ERC, Ocean, 101071601).


\clearpage
\bibliography{references}

@misc{komo_shill-proof_2024,
    title = {Shill-{Proof} {Auctions}},
    url = {http://arxiv.org/abs/2404.00475},
    doi = {10.48550/arXiv.2404.00475},
    urldate = {2025-07-10},
    publisher = {arXiv},
    author = {Komo, Andrew and Kominers, Scott Duke and Roughgarden, Tim},
    month = nov,
    year = {2024},
    note = {arXiv:2404.00475 [econ]},
}

@inproceedings{yang_adversarial_2020,
    title = {Adversarial {Bandits} with {Corruptions}: {Regret} {Lower} {Bound} and {No}-regret {Algorithm}},
    volume = {33},
    url = {https://proceedings.neurips.cc/paper_files/paper/2020/file/e655c7716a4b3ea67f48c6322fc42ed6-Paper.pdf},
    booktitle = {Advances in {Neural} {Information} {Processing} {Systems}},
    publisher = {Curran Associates, Inc.},
    author = {yang, lin and Hajiesmaili, Mohammad and Talebi, Mohammad Sadegh and Lui, John C. S. and Wong, Wing Shing},
    editor = {Larochelle, H. and Ranzato, M. and Hadsell, R. and Balcan, M. F. and Lin, H.},
    year = {2020},
    pages = {19943--19952},
}

@misc{wang_learning_2023,
    title = {Learning against {Non}-credible {Auctions}},
    url = {http://arxiv.org/abs/2311.15203},
    doi = {10.48550/arXiv.2311.15203},
    urldate = {2025-10-24},
    publisher = {arXiv},
    author = {Wang, Qian and Xia, Xuanzhi and Yang, Zongjun and Deng, Xiaotie and Kong, Yuqing and Zhang, Zhilin and Wang, Liang and Yu, Chuan and Xu, Jian and Zheng, Bo},
    month = nov,
    year = {2023},
    note = {arXiv:2311.15203 [cs]},
}

@article{akbarpour_credible_2020,
    title = {Credible {Auctions}: {A} {Trilemma}},
    volume = {88},
    issn = {0012-9682},
    shorttitle = {Credible {Auctions}},
    url = {https://www.econometricsociety.org/doi/10.3982/ECTA15925},
    doi = {10.3982/ECTA15925},
    language = {en},
    number = {2},
    urldate = {2025-10-20},
    journal = {Econometrica},
    author = {Akbarpour, Mohammad and Li, Shengwu},
    year = {2020},
    pages = {425--467},
}

@article{han2020optimal,
  title={Optimal no-regret learning in repeated first-price auctions},
  author={Han, Yanjun and Zhou, Zhengyuan and Weissman, Tsachy},
  journal={arXiv preprint arXiv:2003.09795},
  year={2020}
}

@misc{zhang_leveraging_2022,
    title = {Leveraging the {Hints}: {Adaptive} {Bidding} in {Repeated} {First}-{Price} {Auctions}},
    shorttitle = {Leveraging the {Hints}},
    url = {http://arxiv.org/abs/2211.06358},
    doi = {10.48550/arXiv.2211.06358},
    urldate = {2025-11-07},
    publisher = {arXiv},
    author = {Zhang, Wei and Han, Yanjun and Zhou, Zhengyuan and Flores, Aaron and Weissman, Tsachy},
    month = nov,
    year = {2022},
    note = {arXiv:2211.06358 [cs]},
}

@misc{gupta_better_2019,
    title = {Better {Algorithms} for {Stochastic} {Bandits} with {Adversarial} {Corruptions}},
    url = {http://arxiv.org/abs/1902.08647},
    doi = {10.48550/arXiv.1902.08647},
    urldate = {2025-11-07},
    publisher = {arXiv},
    author = {Gupta, Anupam and Koren, Tomer and Talwar, Kunal},
    month = mar,
    year = {2019},
    note = {arXiv:1902.08647 [cs]},
}

@misc{alon_online_2015,
    title = {Online {Learning} with {Feedback} {Graphs}: {Beyond} {Bandits}},
    shorttitle = {Online {Learning} with {Feedback} {Graphs}},
    url = {http://arxiv.org/abs/1502.07617},
    doi = {10.48550/arXiv.1502.07617},
    urldate = {2026-02-12},
    publisher = {arXiv},
    author = {Alon, Noga and Cesa-Bianchi, Nicolò and Dekel, Ofer and Koren, Tomer},
    month = feb,
    year = {2015},
    note = {arXiv:1502.07617 [cs]},
}

@misc{lykouris_stochastic_2018,
    title = {Stochastic bandits robust to adversarial corruptions},
    url = {http://arxiv.org/abs/1803.09353},
    doi = {10.48550/arXiv.1803.09353},
    urldate = {2026-02-16},
    publisher = {arXiv},
    author = {Lykouris, Thodoris and Mirrokni, Vahab and Leme, Renato Paes},
    month = mar,
    year = {2018},
    note = {arXiv:1803.09353 [cs]},
}

@article{lu_stochastic_2021,
    title = {Stochastic {Graphical} {Bandits} with {Adversarial} {Corruptions}},
    volume = {35},
    issn = {2374-3468, 2159-5399},
    url = {https://ojs.aaai.org/index.php/AAAI/article/view/17060},
    doi = {10.1609/aaai.v35i10.17060},
    number = {10},
    urldate = {2026-02-16},
    journal = {Proceedings of the AAAI Conference on Artificial Intelligence},
    author = {Lu, Shiyin and Wang, Guanghui and Zhang, Lijun},
    month = may,
    year = {2021},
    pages = {8749--8757},
}

@article{fan2024policy,
  title={Policy optimization using semiparametric models for dynamic pricing},
  author={Fan, Jianqing and Guo, Yongyi and Yu, Mengxin},
  journal={Journal of the American Statistical Association},
  volume={119},
  number={545},
  pages={552--564},
  year={2024},
  publisher={Taylor \& Francis}
}

@article{javanmard2019dynamic,
  title={Dynamic pricing in high-dimensions},
  author={Javanmard, Adel and Nazerzadeh, Hamid},
  journal={Journal of Machine Learning Research},
  volume={20},
  number={9},
  pages={1--49},
  year={2019}
}

@misc{balseiroContextualBanditsCrosslearning2021,
    title = {Contextual {Bandits} with {Cross}-learning},
    url = {http://arxiv.org/abs/1809.09582},
    doi = {10.48550/arXiv.1809.09582},
    urldate = {2026-04-26},
    publisher = {arXiv},
    author = {Balseiro, Santiago and Golrezaei, Negin and Mahdian, Mohammad and Mirrokni, Vahab and Schneider, Jon},
    month = nov,
    year = {2021},
    note = {arXiv:1809.09582 [cs]},
}

@inproceedings{achddouFastRateLearning2021,
    series = {Proceedings of machine learning research},
    title = {Fast rate learning in stochastic first price bidding},
    volume = {157},
    url = {https://proceedings.mlr.press/v157/achddou21a.html},
    booktitle = {Proceedings of the 13th asian conference on machine learning},
    publisher = {PMLR},
    author = {Achddou, Juliette and Cappé, Olivier and Garivier, Aurélien},
    editor = {Balasubramanian, Vineeth N. and Tsang, Ivor},
    month = nov,
    year = {2021},
    pages = {1754--1769},
}

@inproceedings{kleinbergValueKnowingDemand2003,
    address = {Cambridge, MA, USA},
    title = {The value of knowing a demand curve: bounds on regret for online posted-price auctions},
    isbn = {978-0-7695-2040-7},
    shorttitle = {The value of knowing a demand curve},
    url = {http://ieeexplore.ieee.org/document/1238232/},
    doi = {10.1109/SFCS.2003.1238232},
    urldate = {2024-10-18},
    booktitle = {44th {Annual} {IEEE} {Symposium} on {Foundations} of {Computer} {Science}, 2003. {Proceedings}.},
    publisher = {IEEE Computer. Soc},
    author = {Kleinberg, R. and Leighton, T.},
    year = {2003},
    pages = {594--605},
}

@article{besbesDynamicPricingKnowing2009,
    title = {Dynamic {Pricing} {Without} {Knowing} the {Demand} {Function}: {Risk} {Bounds} and {Near}-{Optimal} {Algorithms}},
    volume = {57},
    issn = {0030-364X, 1526-5463},
    shorttitle = {Dynamic {Pricing} {Without} {Knowing} the {Demand} {Function}},
    url = {https://pubsonline.informs.org/doi/10.1287/opre.1080.0640},
    doi = {10.1287/opre.1080.0640},
    language = {en},
    number = {6},
    urldate = {2026-05-02},
    journal = {Operations Research},
    author = {Besbes, Omar and Zeevi, Assaf},
    month = dec,
    year = {2009},
    pages = {1407--1420},
}

@inproceedings{kleinbergNearlyTightBounds2004,
    title = {Nearly tight bounds for the continuum-armed bandit problem},
    url = {https://api.semanticscholar.org/CorpusID:15364316},
    booktitle = {Neural information processing systems},
    author = {Kleinberg, Robert D.},
    year = {2004},
}

@inproceedings{cesa-bianchiDynamicPricingFinitely2019,
    series = {Proceedings of machine learning research},
    title = {Dynamic pricing with finitely many unknown valuations},
    volume = {98},
    url = {https://proceedings.mlr.press/v98/cesa-bianchi19a.html},
    booktitle = {Proceedings of the 30th international conference on algorithmic learning theory},
    publisher = {PMLR},
    author = {Cesa-Bianchi, Nicolò and Cesari, Tommaso and Perchet, Vianney},
    editor = {Garivier, Aurélien and Kale, Satyen},
    month = mar,
    year = {2019},
    pages = {247--273},
}

@misc{alonBanditsExpertsTale2013,
    title = {From {Bandits} to {Experts}: {A} {Tale} of {Domination} and {Independence}},
    shorttitle = {From {Bandits} to {Experts}},
    url = {http://arxiv.org/abs/1307.4564},
    doi = {10.48550/arXiv.1307.4564},
    urldate = {2026-02-13},
    publisher = {arXiv},
    author = {Alon, Noga and Cesa-Bianchi, Nicolò and Gentile, Claudio and Mansour, Yishay},
    month = jul,
    year = {2013},
    note = {arXiv:1307.4564 [cs]},
}

@inproceedings{erezBestofallworldsOnlineLearning2021,
    title = {Towards best-of-all-worlds online learning with feedback graphs},
    volume = {34},
    url = {https://proceedings.neurips.cc/paper_files/paper/2021/file/ef8f94395be9fd78b7d0aeecf7864a03-Paper.pdf},
    booktitle = {Advances in neural information processing systems},
    publisher = {Curran Associates, Inc.},
    author = {Erez, Liad and Koren, Tomer},
    editor = {Ranzato, M. and Beygelzimer, A. and Dauphin, Y. and Liang, P.S. and Vaughan, J. Wortman},
    year = {2021},
    pages = {28511--28521},
}

@incollection{boucheron2003concentration,
  title={Concentration inequalities},
  author={Boucheron, St{\'e}phane and Lugosi, G{\'a}bor and Bousquet, Olivier},
  booktitle={Summer school on machine learning},
  pages={208--240},
  year={2003},
  publisher={Springer}
}
\bibliographystyle{abbrvnat}

\clearpage

\appendix

\crefalias{section}{appendix}
\crefalias{subsection}{appendix}
\crefalias{subsubsection}{appendix}

\section{Proofs of Section \ref{sec:model}}
\label{app:proofs-model}

\subsection{Proof of Lemma \ref{lem:rhr-decomposition}}
\label{app:proof-of-rhr-decomposition}

\rhrdecomposition*

\begin{proof}
    First, we can write the PDF of $O$ for any $p \in [0, 1]$ as
    \begin{multline*}
        f_O(p) 
        = \P{O = p} 
        = \P{\max\{B, S\} = p}
        \\ = \P{B = p \wedge S < p} + \P{S = p \wedge B < p}
        = f_B(p) F_S(p) + f_S(p) F_B(p),
    \end{multline*}
    then, because $F_O(p) = F_B(p)F_S(p)$,
    \[
        r_O(p)
        = \frac{f_O(p)}{F_O(p)}
        = \frac{f_B(p) F_S(p) + f_B(p) F_S(p)}{F_B(p)F_S(p)}
        = r_B(p) + r_S(p).
        \qedhere
    \]
\end{proof}

\subsection{Proof of Lemma \ref{lem:optima-under-shilling}}
\label{app:proof-of-optima-under-shilling}

\optimaundershilling*
    
\begin{proof}
For all $p\in[0,1]$ we have $u_O(p)=u_B(p)F_S(p)$. Let $p_B^{\max}=\max A_B$. For any $p\le p_B^{\max}$,
\[
u_B(p)\le u_B(p_B^{\max}),\qquad F_S(p)\le F_S(p_B^{\max}),
\]
because $F_S$ is non-decreasing, hence
\[
u_O(p)=u_B(p)F_S(p)\le u_B(p_B^{\max})F_S(p_B^{\max})=u_O(p_B^{\max}).
\]
Thus no $p\le p_B^{\max}$ can strictly beat $p_B^{\max}$ under $u_O$. If there exists $p>p_B^{\max}$ with $u_O(p)>u_O(p_B^{\max})$, then every maximizer of $u_O$ lies in $(p_B^{\max},1]$ and $p_O^{\max}>p_B^{\max}$. Otherwise, $p_B^{\max}\in A_O$ and therefore $p_O^{\max}\ge p_B^{\max}$. This proves $p_O^{\max}\ge p_B^{\max}$.

For the second claim, fix $p\in[0,p_B^{\max}]$ and assume $p\in A_O$. Then $u_O(p)=\max_{q}u_O(q)\ge u_O(p_B^{\max})$, while the inequality above gives $u_O(p)\le u_O(p_B^{\max})$. Hence $u_O(p)=u_O(p_B^{\max})$, i.e.
\[
u_B(p)F_S(p)=u_B(p_B^{\max})F_S(p_B^{\max}).
\]
Since $u_B(p)\le u_B(p_B^{\max})$ and $F_S(p)\le F_S(p_B^{\max})$, equality implies $u_B(p)=u_B(p_B^{\max})$ and $F_S(p)=F_S(p_B^{\max})$. The first identity yields $p\in A_B$, proving the stated implication.
\end{proof}

\subsection{Proof of Lemma \ref{lem:single-optima-under-shilling}}
\label{app:proof-of-single-optima-under-shilling}

\singleoptimaundershilling*
\begin{proof}
By \cref{lem:optima-under-shilling}, $ p_O^{\max}\ge p_B^*$. It remains to rule out equality. Since $ p_B^*$ is an interior maximizer of $ u_B$, we have $ u_B'(p_B^*) = 0$. Also $u_O(p)=u_B(p)F_S(p)$,
so
\[
u_O'(p_B^*)
=
u_B'(p_B^*)F_S(p_B^*)+u_B(p_B^*)f_S(p_B^*)
=
u_B(p_B^*)f_S(p_B^*).
\]
At a nontrivial optimum, $ u_B(p_B^*)>0$. Since $ f_S(p_B^*)>0$, we get $ u_O'(p_B^*)>0$. Therefore $ p_B^*$ cannot be a maximizer of $ u_O$, and since $ p_O^{\max}\ge p_B^*$, we must have $ p_O^{\max}>p_B^*$.
\end{proof}



\subsection{Proof of Lemma \ref{lem:mixture-optimum}}
\label{app:proof-of-mixture-optimum}

\mixtureoptimum*

\begin{proof}
Fix \(v\in[0,1]\). For every \(\lambda\in[0,1]\), recall that
\[
    F_\lambda(p)
    =
    \lambda F_B(p)+(1-\lambda)F_O(p)
    =
    F_B(p)G_\lambda(p),
    \qquad
    G_\lambda(p)\coloneq \lambda+(1-\lambda)F_S(p).
\]
Hence \(u_\lambda(p)=u_B(p)G_\lambda(p)\), and \(G_\lambda\) is nondecreasing because \(F_S\) is nondecreasing.

We first prove that \(p_\lambda^{\max}(v)\ge p_B^{\max}(v)\). This is exactly the same comparison as in \cref{lem:optima-under-shilling}, with \(F_S\) replaced by \(G_\lambda\). Indeed, for every \(p\le p_B^{\max}(v)\), we have \(u_B(p)\le u_B(p_B^{\max}(v))\) by definition of \(p_B^{\max}(v)\), and \(G_\lambda(p)\le G_\lambda(p_B^{\max}(v))\) because \(G_\lambda\) is nondecreasing. Therefore
\[
    u_\lambda(p)
    =
    u_B(p)G_\lambda(p)
    \le
    u_B(p_B^{\max}(v))G_\lambda(p_B^{\max}(v))
    =
    u_\lambda(p_B^{\max}(v)).
\]
Thus no point to the left of \(p_B^{\max}(v)\) can strictly beat \(p_B^{\max}(v)\) under \(u_\lambda\). It follows that the rightmost maximizer of \(u_\lambda\) satisfies \(p_\lambda^{\max}(v)\ge p_B^{\max}(v)\).

It remains to prove monotonicity in \(\lambda\). Let \(0\le \lambda_1\le \lambda_2\le 1\). We will show that \(p_{\lambda_1}^{\max}(v)\ge p_{\lambda_2}^{\max}(v)\). Since
\[
    u_{\lambda_1}(p)
    =
    u_{\lambda_2}(p)
    \frac{G_{\lambda_1}(p)}{G_{\lambda_2}(p)},
\]
it is enough to show that the ratio \(G_{\lambda_1}(p)/G_{\lambda_2}(p)\) is nondecreasing in \(p\). Write \(x=F_S(p)\in[0,1]\) and define
\[
    R(x)
    \coloneq
    \frac{\lambda_1+(1-\lambda_1)x}
         {\lambda_2+(1-\lambda_2)x}.
\]
The denominator is positive whenever \(G_{\lambda_2}(p)>0\); on points where it is zero, both \(G_{\lambda_1}\) and \(G_{\lambda_2}\) are zero and the value of the ratio is irrelevant for the objective. On the positive region, differentiating gives
\[
    R'(x)
    =
    \frac{(1-\lambda_1)(\lambda_2+(1-\lambda_2)x)
    -(1-\lambda_2)(\lambda_1+(1-\lambda_1)x)}
    {(\lambda_2+(1-\lambda_2)x)^2}
    =
    \frac{\lambda_2-\lambda_1}
    {(\lambda_2+(1-\lambda_2)x)^2}
    \ge 0.
\]
Since \(F_S\) is nondecreasing in \(p\), the map \(p\mapsto G_{\lambda_1}(p)/G_{\lambda_2}(p)\) is also nondecreasing.

Now apply the same rightmost-maximizer comparison again, but with \(u_{\lambda_2}\) as the baseline objective and \(G_{\lambda_1}/G_{\lambda_2}\) as the nondecreasing multiplier. For every \(p\le p_{\lambda_2}^{\max}(v)\), we have \(u_{\lambda_2}(p)\le u_{\lambda_2}(p_{\lambda_2}^{\max}(v))\), and we also have
\[
    \frac{G_{\lambda_1}(p)}{G_{\lambda_2}(p)}
    \le
    \frac{G_{\lambda_1}(p_{\lambda_2}^{\max}(v))}
         {G_{\lambda_2}(p_{\lambda_2}^{\max}(v))}.
\]
Therefore \(u_{\lambda_1}(p)\le u_{\lambda_1}(p_{\lambda_2}^{\max}(v))\) for every \(p\le p_{\lambda_2}^{\max}(v)\). Hence no point to the left of \(p_{\lambda_2}^{\max}(v)\) can strictly beat \(p_{\lambda_2}^{\max}(v)\) under \(u_{\lambda_1}\), and so \(p_{\lambda_1}^{\max}(v)\ge p_{\lambda_2}^{\max}(v)\).

Taking \(\lambda_2=1\) gives \(p_\lambda^{\max}(v)\ge p_1^{\max}(v)=p_B^{\max}(v)\), and taking \(\lambda_1=0\), \(\lambda_2=\lambda\), gives \(p_0^{\max}(v)=p_O^{\max}(v)\ge p_\lambda^{\max}(v)\). Therefore
\[
    p_B^{\max}(v)
    =
    p_1^{\max}(v)
    \le
    p_\lambda^{\max}(v)
    \le
    p_0^{\max}(v)
    =
    p_O^{\max}(v),
\]
as claimed.
\end{proof}

\section{Algorithm for Known $F_s$}\label{sec:algorithm}

\Cref{alg:known_FS} details the elimination procedure that allows to recover the regret rates described in \cref{thm:upper-bound}. The algorithm runs in epochs. The variable \(t\) denotes global time, \(n_m\) the instant of time at which \(m\) starts, and \(\tau_m\) counts the rounds elapsed within that epoch. A different grid of prices is considered during each epoch, specifically during epoch $m$, the grid $\cG_m$ counts a total of $2^m$ evenly-spaced points chosen in the active set region $\cA^m(v)$, and it holds that $\cG_m \subset \cG_{m+1}$.  For each \(v\in\mathcal V\), the active set \(A_m(v)\subseteq[0,v]\) is a continuous interval, initially identified with the whole space. All sampling, estimation, and elimination comparisons are performed on the finite set \(A_m(v)\cap\mathcal G_m\).

Throughout the interaction, the algorithm maintains two different certificates, a robust one, \(\mathsf C^{\rob}_t=(N_t,\widehat F^{\rob}_t,\widehat\Delta^{\rob}_t,r_{\rob}(t,\cdot))\), which relies only on direct win/loss observations; and an optimistic one \(\mathsf C^{\opt}_t=(G_t,z_t,\widehat F^{\opt}_t,\widehat\Delta^{\opt}_t,r_{\opt}(t,\cdot))\) that uses the weighted least-squares construction based on direct and suffix-difference measurements.

For the gaps we use the convention \(\Delta(v,q,p)=u_B(v,q)-u_B(v,p)\), such that a positive gap implies that \(q\) improves over \(p\).
At the end of the epoch, the indicator \(\iota_m\) records which certificate triggered elimination: \(\iota_m=1\) for the optimistic certificate and \(\iota_m=0\) for the robust certificate. 

\subsection{Dyadic validation over the shill scale}
\label{sec:dyadic-gamma}

At epoch \(m\), the grid \(\mathcal G_m\) has mesh \(h_m\asymp 2^{-m}\). The algorithm therefore does not race over price grids. It only validates the unknown shill scale through the dyadic set \( \Gamma=\{2^{-\ell}:0\le \ell\le \lceil\log_2 T\rceil\} \). For every \(\bar\gamma\in\Gamma\), define
\[
    \omega_{m,\bar\gamma}
    \coloneq 
    \frac{\bar\gamma}{C_{\suf}(2^{-m}+h_m)}.
\]
The optimistic routine uses direct and suffix statistics, using candidate \(\bar\gamma\) define
\[
    G_t^{\bar\gamma}
    =
    G_t^{\dir}+\omega_{m,\bar\gamma}G_t^{\suf},
    \qquad
    z_t^{\bar\gamma}
    =
    z_t^{\dir}+\omega_{m,\bar\gamma}z_t^{\suf}
    \quad \text{and} \quad
    \widehat F_t^{\opt,\bar\gamma}
    =
    (G_t^{\bar\gamma})^\dagger z_t^{\bar\gamma}.
\]
For a fixed validation constant \(C_{\val}\ge 2C_S\), define the admissible candidate set at epoch \(m\) by
\begin{multline*}
\Gamma_m^{\adm}
\coloneq 
\Big\{
    \bar\gamma\in\Gamma:
    F_S(q) \ge \bar\gamma \text{ and } F_S(q_{j+1})-F_S(q_j) \le C_{\val}\bar\gamma h_m
\\
    \text{ for every }v\in\cV \text{ and every } q_j,q_{j+1}\in A_m(v)\cap\cG_m
\Big\}.
\end{multline*}
Only candidates in \(\Gamma_m^{\adm}\) are allowed to validate. Candidates outside
\(\Gamma_m^{\adm}\) are ignored by the optimistic validation step, equivalently their
optimistic radius is set to \(+\infty\).

\begin{lemma}[Dyadic admissible shill-scale candidate]
\label{lem:dyadic-gamma}
Fix an epoch \(m\), and suppose that \(F_S\) is \((\gamma,C_S)\)-regular on the active regions considered by the optimistic branch. Let \(C_{\val}\ge 2C_S\). If \(\gamma\ge 1/T\), then there exists \(\bar\gamma\in\Gamma_m^{\adm}\) such that
\[
    \bar\gamma\le \gamma\le 2\bar\gamma.
\]
For this candidate,
\[
    \frac12\omega_m^\star
    \le
    \omega_{m,\bar\gamma}
    \le
    \omega_m^\star,
    \qquad
    \omega_m^\star
    \coloneq 
    \frac{\gamma}{C_{\suf}(2^{-m}+h_m)}.
\]
\end{lemma}

\begin{proof}
Since \(\Gamma\) is dyadic and contains all powers \(2^{-\ell}\) down to \(1/T\), for every
\(\gamma\in[1/T,1]\) there exists \(\bar\gamma\in\Gamma\) with
\(\bar\gamma\le\gamma\le2\bar\gamma\). By \((\gamma,C_S)\)-regularity,
\(F_S(q)\ge\gamma\ge\bar\gamma\) on the active grid. For every adjacent pair
\(q_j,q_{j+1}\),
\[
F_S(q_{j+1})-F_S(q_j)
\le C_S\gamma h_m
\le 2C_S\bar\gamma h_m
\le C_{\val}\bar\gamma h_m.
\]
Thus \(\bar\gamma\in\Gamma_m^{\adm}\). Dividing
\(\bar\gamma\le\gamma\le2\bar\gamma\) by \(C_{\suf}(2^{-m}+h_m)\) gives the displayed comparison between \(\omega_{m,\bar\gamma}\) and \(\omega_m^\star\).
\end{proof}
Next, we provide a closed expression for $B_t^{\bar \gamma}(v)$, which is computed by the algorithm. For candidate \(\bar\gamma\in\Gamma_m^{\adm}\), let \(\mathcal S_t^{\dir}\) be the set of direct measurements collected up to time \(t\), and let \(\mathcal S_t^{\suf}\) be the set of suffix-difference measurements collected up to time \(t\). For each measurement \(s\in\mathcal S_t^{\dir}\cup\mathcal S_t^{\suf}\), define
\[
\alpha_s^{\bar\gamma}
=
\begin{cases}
1, & s\in\mathcal S_t^{\dir},\\
\omega_{m,\bar\gamma}, & s\in\mathcal S_t^{\suf},
\end{cases}
\qquad
R_s^{\bar\gamma}
=
\begin{cases}
1, & s\in\mathcal S_t^{\dir},\\
4/\bar\gamma, & s\in\mathcal S_t^{\suf}.
\end{cases}
\]
Then
\begin{equation}
\label{eq:range-cert-gamma}
B_t^{\bar\gamma}(v)
\coloneq 
\max_{q_i,q_j\in A_m(v)\cap\cG_m}
\max_{s\in\mathcal S_t^{\dir}\cup\mathcal S_t^{\suf}}
\alpha_s^{\bar\gamma}
\left|
g_{v,i,j}^{\top}
(G_t^{\bar\gamma})^\dagger
\Phi_s^\top
\right|
R_s^{\bar\gamma}.
\end{equation}

\subsection{Estimate of the average maximal confidence radius}
The algorithm maintains two quantities $\hat{r}_{\text{opt}}(t)$ and $\hat{r}_{\text{rob}}(t)$, which estimate the average over values $v$ of the largest confidence interval around the gap estimates $\widehat{\Delta}^{\text{opt}}$ and $\widehat{\Delta}^{\text{rob}}$ respectively.
During $m+1$ epoch, denoting $n_m$ the instant of time when the previous epoch ended, they are formally defined as 
\begin{align}\label{eq:har_r_rob}
    \widehat r_{\rob}(t) = \frac{1}{t-n_m}\sum_{s=n_m}^t \max_{q \in \cA^m(v_s)\cap \cG_m}\sqrt{\frac{2\log(2T / \delta)}{N_t(q)}}
\end{align}
where $N_t(q)$ is the counter of the amount of times price $q$ was played during the epoch. Similarly
\begin{align}\label{eq:har_r_opt}
  \widehat r_{\opt}(t) = \frac{1}{t-n_m}\sum_{s=n_m}^t \max_{q \in \cA^m(v_s)\cap \cG_m}\sqrt{2\log(2T |\Gamma| / \delta)Q^{\bar\gamma}_t(v)}
            +\frac23\log(2T |\Gamma| / \delta)B^{\bar\gamma}_t(v)
\end{align}
where $B_t^{\bar\gamma}(v)$ as in \cref{eq:range-cert-gamma} and \(Q^{\bar\gamma}_t(v) = \max_{q_i,q_j\in A_m(v)\cap\cG_m} g_{v,i,j}^{\top}(G^{\bar\gamma})^\dagger g_{v,i,j}\)\;

The need to average over time allows the algorithm to correctly weight each context depending on the frequency with which the learner observes it.

\begin{algorithm}
\caption{Shill-proof algorithm for known \(F_S\)}
\label{alg:known_FS}
\DontPrintSemicolon
\SetKwInOut{Input}{Input}
\SetKwInOut{Setup}{Setup}

\Input{\(T\), \(\delta\), value grid \(\cV\)}
\Setup{
    \(t\gets 0\). For every \(v\in\cV\), \(A_0(v)\gets [0,v]\). \\
    \(\mathsf C^{\rob}_0\gets (N_0,\widehat F^{\rob}_0,\widehat\Delta^{\rob}_0,r_{\rob}(0,\cdot))\) with all entries equal to \(0\). \\
    \(\widehat r_{\rob}(0)\gets 0\).
}

\BlankLine

\For{\(m=0,\ldots,\lceil\log_2 T\rceil-1\)}{
    \(n_m\gets t\), \(\tau_m\gets 0\), and \(\iota_m\gets 0\)\;
    Construct the epoch grid \(\cG_m\) with mesh \(h_m\asymp 2^{-m}\)\;
    Construct the dyadic candidate set \(\Gamma=\{2^{-\ell}:0\le \ell\le \lceil\log_2 T\rceil\}\)\;
    \(\mathsf C^{\opt}_m\gets (G^{\dir}_m,z^{\dir}_m,G^{\suf}_m,z^{\suf}_m)\) with all entries equal to \(0\)\;
    
    \While{\(t<T\)}{
        Observe \(\tilde v_t\) and round it to the closest \(v_t\in\cV\)\;
        \(\mathsf{valid}^{\opt}_t \gets \mathrm{false}\)\;
        \(\tau_m\gets\tau_m+1\)\;
        
        \uIf{\(t\) is even}{
            Post \(p_t\sim \mathrm{Unif}(A_m(v_t)\cap\cG_m)\) and observe \(\ind{p_t\geq b_t},o_t\)\;
            \((\mathsf C^{\opt}_m,\bar\gamma_m^\star,\mathsf{valid}^{\opt}_t) \gets \textsc{OptRaceCertificate}(m,A_m,v_t,p_t,\ind{p_t\geq b_t},o_t,\Gamma,\mathsf C^{\opt}_m)\)\;
            \(\mathsf C^{\rob}_t\gets \mathsf C^{\rob}_{t-1}\)\;
        }
        \Else{
            Post \(p_t\in\argmin_{p\in A_m(v_t)\cap\cG_m}N_{t-1}(p)\) and observe \(Y^{\rob}_t\gets\ind{p_t\ge b_t}\)\;
            \(\mathsf C^{\rob}_t\gets
            \textsc{RobCertificate}(A_m,v_t,p_t,Y^{\rob}_t,\mathsf C^{\rob}_{t-1})\)\;
        } 
        
        \(\widehat r_{\rob}(\tau_m)\gets \tau_m^{-1}\sum_{s=n_m}^{t}r_{\rob}(t,v_s)\)\;
        
        \If{\(\mathsf{valid}^{\opt}_t=\mathrm{true}\)}{
            \(\iota_m\gets 1\), \(t\gets t+1\) and \textbf{break}\;
        }
        
        \If{\(\widehat r_{\rob}(\tau_m)+\sqrt{\log(2T/\delta)/(2\tau_m)}\le 2^{-(m+1)}\)}{
            \(\iota_m\gets 0\), \(t\gets t+1\) and \textbf{break}\;
        }
        
        \(t\gets t+1\)\;
    }
    
    \uIf{\(\iota_m=1\)}{
        \ForEach{\(v\in\cV\)}{
            \(\begin{aligned}
                A_{m+1}(v)
            \gets
                \Conv\{
                    & p\in A_m(v)\cap\cG_m:
                    \\ & \widehat\Delta^{\opt,\bar\gamma_m^\star}_t(v,q,p)\le 2^{-(m+1)}
                \text{ for all }q\in A_m(v)\cap\cG_m
                \}
            \end{aligned}\)\;
        }
    }
    \Else{
        \ForEach{\(v\in\cV\)}{
            \(\begin{aligned}
                A_{m+1}(v)
            \gets
                \Conv\{
                    & p\in A_m(v)\cap\cG_m:
                    \\ & \widehat\Delta^{\rob}_t(v,q,p)\le 2^{-(m+1)}
                \text{ for all }q\in A_m(v)\cap\cG_m
                \}
            \end{aligned}\)\;
        }
    }
}
\end{algorithm}

\begin{algorithm}[h]
\caption{\(\textsc{RobCertificate}\)}
\label{alg:update-robust-certificate}
\DontPrintSemicolon
\KwIn{Active sets \(A_m\), value \(v_t\), bid \(p_t\), direct feedback \(Y^{\rob}_t\), current certificate \(\mathsf C^{\rob}_{t-1}\)}
\KwOut{Updated certificate \(\mathsf C^{\rob}_t=(N_t,\widehat F^{\rob}_t,\widehat\Delta^{\rob}_t,r_{\rob}(t,\cdot))\)}

\(N_t(p_t)\gets N_{t-1}(p_t)+1\), and \(N_t(q)\gets N_{t-1}(q)\) for all \(q\neq p_t\)\;

\(\widehat F^{\rob}_t(p_t)\gets \bigl(N_{t-1}(p_t)\widehat F^{\rob}_{t-1}(p_t)+Y^{\rob}_t\bigr)/N_t(p_t)\)\;

\(\widehat F^{\rob}_t(q)\gets \widehat F^{\rob}_{t-1}(q)\) for all \(q\neq p_t\)\;

\ForEach{\(v\in\cV\) and \(p,q\in A_m(v)\cap\cG_m\)}{
    \(\widehat\Delta^{\rob}_t(v,q,p)\gets (v-q)\widehat F^{\rob}_t(q)-(v-p)\widehat F^{\rob}_t(p)\)\;
}

\ForEach{\(v\in\cV\)}{
    \(r_{\rob}(t,v)\gets \max_{q\in A_m(v)\cap\cG_m}\sqrt{2\log(2T/\delta)/N_t(q)}\)\;
}

\Return{\(\mathsf C^{\rob}_t\)}\;
\end{algorithm}

\begin{algorithm}
\caption{\(\textsc{OptRaceCertificate}\)}
\label{alg:update-optimistic-race-certificate}
\DontPrintSemicolon
\KwIn{Epoch \(m\), active sets \(A_m\), value \(v_t\), bid \(p_t\), feedback \(b_t,o_t\), candidates \(\Gamma\), current optimistic statistics \(\mathsf C^{\opt}_m\)}
\KwOut{Updated optimistic statistics \(\mathsf C^{\opt}_m\), validating candidate \(\bar\gamma^\star\), flag \(\mathsf{valid}^{\opt}\)}

\(G^{\dir}\gets G^{\dir}\), \(z^{\dir}\gets z^{\dir}\), \(G^{\suf}\gets G^{\suf}\), and \(z^{\suf}\gets z^{\suf}\)\;

\(Y^{\dir}_t\gets \ind{p_t\ge b_t}\) and \(\Phi^{\dir}_t\gets e_{p_t}^{\top}\)\;
\(G^{\dir}\gets G^{\dir}+(\Phi^{\dir}_t)^\top\Phi^{\dir}_t\) and \(z^{\dir}\gets z^{\dir}+(\Phi^{\dir}_t)^\top Y^{\dir}_t\)\;

\ForEach{\(q\in A_m(v_t)\cap\cG_m\) with \(q\ge p_t\)}{
    \(Y_t(q)\gets \ind{p_t\ge b_t}+\ind{p_t<b_t}\ind{o_t\le q}/F_S(q)\)\;
}

Write \(A_m(v_t)\cap\cG_m=\{q_1<\cdots<q_k\}\)\;

\ForEach{adjacent pair \(q_j,q_{j+1}\) with \(q_j\ge p_t\)}{
    \(Y^{\suf}_{t,j}\gets (v_t-q_{j+1})Y_t(q_{j+1})-(v_t-q_j)Y_t(q_j)\)\;
    \(\Phi^{\suf}_{t,j}\gets (v_t-q_{j+1})e_{q_{j+1}}^\top-(v_t-q_j)e_{q_j}^\top\)\;
    \(G^{\suf}\gets G^{\suf}+(\Phi^{\suf}_{t,j})^\top\Phi^{\suf}_{t,j}\)\;
    \(z^{\suf}\gets z^{\suf}+(\Phi^{\suf}_{t,j})^\top Y^{\suf}_{t,j}\)\;
}

\(\mathsf{valid}^{\opt}\gets\mathrm{false}\), \(\bar\gamma^\star\gets\bot\)\;

\ForEach{\(\bar\gamma\in\Gamma_m^{\adm}\)}{
    \(\omega_{m,\bar\gamma}\gets \bar\gamma/(C_{\suf}(2^{-m}+h_m))\)\;
    \(G^{\bar\gamma}\gets G^{\dir}+\omega_{m,\bar\gamma}G^{\suf}\)\;
    \(z^{\bar\gamma}\gets z^{\dir}+\omega_{m,\bar\gamma}z^{\suf}\)\;
    \(\widehat F^{\opt,\bar\gamma}_t\gets (G^{\bar\gamma})^\dagger z^{\bar\gamma}\)\;
    
    \ForEach{\(v\in\cV\) and \(q_i,q_j\in A_m(v)\cap\cG_m\)}{
        \(g_{v,i,j}\gets (v-q_i)e_{q_i}-(v-q_j)e_{q_j}\)\;
        \(\widehat\Delta^{\opt,\bar\gamma}_t(v,q_i,q_j)\gets g_{v,i,j}^{\top}\widehat F^{\opt,\bar\gamma}_t\)\;
    }
    
    \ForEach{\(v\in\cV\)}{
        \uIf{some \(g_{v,i,j}\notin \mathrm{range}(G^{\bar\gamma})\) for \(q_i,q_j\in A_m(v)\cap\cG_m\)}{
            \(Q^{\bar\gamma}_t(v)\gets+\infty\), \(B^{\bar\gamma}_t(v)\gets+\infty\), and \(r_{\opt,\bar\gamma}(t,v)\gets+\infty\)\;
        }
        \Else{
            \(Q^{\bar\gamma}_t(v)\gets \max_{q_i,q_j\in A_m(v)\cap\cG_m} g_{v,i,j}^{\top}(G^{\bar\gamma})^\dagger g_{v,i,j}\)\;
            \(B_t^{\bar\gamma}(v)\gets\) the range certificate in \cref{eq:range-cert-gamma}\;
            \(r_{\opt,\bar\gamma}(t,v)\gets
            \sqrt{2\log(2T|\Gamma|/\delta)Q^{\bar\gamma}_t(v)}
            +\frac23\log(2T|\Gamma|/\delta)B^{\bar\gamma}_t(v)\)\;
        }
    }
    
    \(\widehat r_{\opt,\bar\gamma}(\tau_m)\gets
    \tau_m^{-1}\sum_{s=n_m}^{t}r_{\opt,\bar\gamma}(t,v_s)\)\;
    
    \If{\(\widehat r_{\opt,\bar\gamma}(\tau_m)+\sqrt{\log(2T|\Gamma|/\delta)/(2\tau_m)}\le 2^{-(m+1)}\)}{
        \(\mathsf{valid}^{\opt}\gets\mathrm{true}\), \(\bar\gamma^\star\gets\bar\gamma\)\;
    }
}

\Return{\((\mathsf C^{\opt}_m,\bar\gamma^\star,\mathsf{valid}^{\opt})\)}\;
\end{algorithm}

\section{Proofs of Section \ref{sec:upper-bound}}\label{sec:proof_upper}

\subsection{Proof of Lemma \ref{lem:rhr-implies-interval-level-sets}}
\label{app:proof-of-rhr-implies-interval-level-sets}


For $p\in(0,1)$,
\[
    u_B'(p) = (v-p) f_B(p) - F_B(p) = F_B(p) \br{(v - p) r_B(p) - 1}.
\]
Since $F_B(p)>0$, we have $u_B'(p)\ge0$ if and only if $v\ge p+1/r_B(p)=\phi_B(p)$. By \cref{ass:weak-rhr-buyer}, $\phi_B$ is nondecreasing on $(0,1)$. Therefore the set $\{p\in(0,1):u_B'(p)\ge0\}=\{p\in(0,1):\phi_B(p)\le v\}$ is an interval of the form $(0,\tau]$, possibly empty or all of $(0,1)$. Hence $u_B$ is weakly increasing up to a cutoff and weakly decreasing afterwards.

Because $F_B$ is continuous on $[0,1]$ and has no atom at $0$, $u_B$ extends continuously to the endpoints, with $u_B(0)=0$ and $u_B(1)\le0$. Adding the endpoints cannot create an additional separated upper-level component. Thus $u_B$ has interval upper level sets on $[0,1]$. Restricting such an interval to a finite grid gives an interval of grid indices.

\subsection{Proof of Theorem \ref{thm:upper-bound}}
Recall that the performance of the algorithm is measured in terms of expected regret, which is defined as 
\begin{equation*}
 R_T=
  \sum_{t=1}^T\lE
  \brs{
    \sup_{p\in[0,1]}u_{\tilde{v}_t}(p)-u_{\tilde{v}_t}(p_t)
  }.
\end{equation*}
Let $p^*$ the unique point that maximises the utility $u_{\tilde{v}_t}(p)$, and $\hat{p^*}= \argmax_{p\in\cG_m}u_{\tilde{v}_t}(p)$, it is therefore possible to decompose the difference in the regret as follows
\begin{align}\label{eq:loss_dec}
    u_{\tilde{v}_t}(p^*)-u_{\tilde{v}_t}(p_t)
 & = \notag
    \underbrace{u_{\tilde{v}_t}(p^*)- u_{\tilde{v}_t}(\hat{p}^*)}_{(i)}
    + \underbrace{u_{\tilde{v}_t}(\hat{p}^*)- u_{v_t}(\hat{p}^*)}_{(ii)}
 \\ & \quad
    + \underbrace{u_{v_t}(\hat{p}^*) - u_{v_t}(p_t)}_{(iii)}
    + \underbrace{u_{v_t}(p_t)- u_{\tilde{v}_t}(p_t)}_{(iv)} \,.
\end{align}
The first term represents the error that the algorithm suffers due to the discretization of the prices with respect to the grid of prices, note that the algorithm uses different grids throughout the interaction. Denote by $q^+(p^*) = \min\{p \in \cG: p\geq p^*\}$ the first point in the grid $\cG_m$ above $p^*$, note that
\begin{align*}
  u_{\tilde{v}_t}(p^*)- u_{\tilde{v}_t}(\hat{p}^*)
  &\leq (\tilde{v}_t-p^*)F_B(p^*) - (\tilde{v}_t-q^+(p^*))F_B(q^+(p^*)) \\
  &\leq (v-p^*)F_B(q^+(p^*)) - (\tilde{v}_t-q^+(p^*))F_B(q^+(p^*))\\
  &= (p^*-q^+(p^*)) \le 1/|\cG_m| \,.
\label{eq:bid-grid-loss-interior}
\end{align*}
where the last inequality follows from the discretization.
Terms $(ii)$ and $(iv)$ correspond to the error caused by the approximation of the received value $\tilde{v}_t$ to $v_t$. Note that for every price $q$, it holds
\begin{equation*}
  |u_{\tilde{v}_t}(q)-u_{v_t}(q)|=|(\tilde{v}_t-v_t)F_B(q)|\le|\tilde{v}_t-v_t| \le \frac{1}{|\cV|}\,,
\end{equation*}
where again the last inequality follows from the fact that the approximation happens with respect to an evenly spaced grid of $| \cV|$ elements.
Replacing both bounds in \cref{eq:loss_dec}, we obtain
\[
 u_{\tilde{v}_t}(p^*)-u_{\tilde{v}_t}(p_t) \leq 
 \frac{1}{| \cG_t|} +\frac{2}{| \cV|} +
 u_{v_t}(\hat{p}^*) - u_{v_t}(p_t)\,,
\]
and by further using this bound in the expression of the expected regret, we obtain 
\begin{align}\label{eq:regre_parts}
    R_T \leq  \frac{2}{| \cV|}T +  \sum_{t=1}^T\lE\brs{\frac{1}{| \cG_t|}+
 u_{v_t}(\hat{p}^*) - u_{v_t}(p_t)} =\frac{2}{| \cV|}T + \sum_{t=1}^T\lE\brs{\frac{1}{| \cG_t|}}+\lE\brs{
 u_{v_t}(\hat{p}^*) - u_{v_t}(p_t)} \,.
\end{align}
To bound the discretization error observe that the grid is fixed within each epoch $m$, during which it consist of $2^m$ prices, therefore
 \begin{align}\label{eq:disc_error}
 \lE\brs{\sum_{t=1}^T\frac{1}{| \cG_t|}} =  \lE\brs{\sum_{m=0}^{\lceil\log T\rceil-1}\sum_{t=n_m+1}^{n_{m+1}} 2^{-m}} \leq \sum_{m=0}^{\lceil\log T\rceil-1} 2^{-m}\lE[n_{m+1}-n_m]\,.
 \end{align}

We focus now on bounding the last term of \cref{eq:regre_parts}, which corresponds to the regret of \cref{alg:known_FS}, which, as for the discretization error, can be decomposed in the different epochs
\begin{align*}
  \sum_{t=1}^T\lE[
 u_{v_t}(\hat{p}^*) - u_{v_t}(p_t)]
 = \sum_{m=0}^{\lceil\log T\rceil-1}\lE\brs{ \sum_{t=n_m+1}^{n_{m+1}}  
 u_{v_t}(\hat{p}^*) - u_{v_t}(p_t)} 
\end{align*}
where we denote by $n_m+1$ the instant of time where the algorithm begins the $m$-th epoch, further considering $n_0=0$ and $n_{\log T}=T$.




The following lemma ensures that during each epoch $m$ the max regret suffered from any still-active arm is at most $2^{-m}$ and the confidence bound of the estimator used for the elimination process at the end of the previous phase.
\begin{lemma}
\label{lem:reg_epoch}
For every epoch \(m\ge 1\), every round
\(t\in\{n_m+1,\ldots,n_{m+1}\}\), and every
\(p\in A_m(v_t)\cap\cG_m\), it holds that
\[
\lE\brs{
u_{v_t}(\hat p^\star_{v_t})-u_{v_t}(p)
\mid \cH_{t-1},v_t
}
\le
2^{-m}
+
\iota_{m-1} r_{\opt,\bar\gamma_{m-1}^\star}(n_m,v_t)
+
(1-\iota_{m-1})r_{\rob}(n_m,v_t)
+
2\delta.
\]
Here \(\iota_{m-1}\) denotes the certificate used to construct \(A_m\) at the end of epoch \(m-1\), and \(\bar\gamma_{m-1}^\star\) is the validating dyadic candidate if \(\iota_{m-1}=1\).
\end{lemma}

Let \( R(t,v_t) := \iota_{m-1} r_{\opt,\bar\gamma_{m-1}^\star}(n_m,v_t) + (1-\iota_{m-1})r_{\rob}(n_m,v_t).\), combining the above result with an application of the tower rule, allows to bound the regret expression as 
\begin{align*}
  \sum_{m=0}^{\lceil\log T\rceil-1}\lE\brs{ \sum_{t=n_m+1}^{n_{m+1}}  
 u_{v_t}(\hat{p}^*) - u_{v_t}(p_t)}
 &= \sum_{m=0}^{\lceil\log T\rceil-1}\lE\brs{ \sum_{t=n_m+1}^{n_{m+1}} \lE\brs{ 
 u_{v_t}(\hat{p}^*) - u_{v_t}(p_t) \mid \cH_{t-1}, v_t}}\\
 &\leq\br{ \sum_{m=0}^{\lceil\log T\rceil-1} \lE\brs{\sum_{t=n_m+1}^{n_{m+1}} 
 2^{-m} + R(t, v_t)}} + 2\delta T\\
 &\leq \br{ \sum_{m=0}^{\lceil\log T\rceil-1} 2^{-m}\lE[n_{m+1}-n_m]+ \lE\brs{\sum_{t=n_m+1}^{n_{m+1}} 
  R(t, v_t)}} + 2\delta T\,.
\end{align*}
The expression can be further simplified with the aid of the following result.
\begin{lemma}
\label{lem:bound_avg_radius}
Let $R(t,v_t)$ be defined as \( R(t,v_t) := \iota_{m-1} r_{\opt,\bar\gamma_{m-1}^\star}(n_m,v_t) + (1-\iota_{m-1})r_{\rob}(n_m,v_t)\), then it holds
\[
\lE\brs{
\sum_{t=n_m+1}^{n_{m+1}} R(t,v_t)
}
\le
2^{-m}
\lE\brs{
    \sum_{k=0}^{\lceil\log_2 T\rceil}
    2^k \mathbb P\!\left(n_{m+1}-n_m\in[2^{k-1},2^k)\mid n_m\right)
}
+
\delta\,
\lE[n_{m+1}-n_m].
\]
\end{lemma}
Hence
\begin{multline*}
  \sum_{m=0}^{\lceil\log T\rceil-1}
  \lE\brs{ \sum_{t=n_m+1}^{n_{m+1}}
  u_{v_t}(\hat p^\star)-u_{v_t}(p_t)}
  \le \\
  \sum_{m=0}^{\lceil\log T\rceil-1}
  2^{-m}
  \left(
  \lE[n_{m+1}-n_m]
  +
  \lE\brs{
  \sum_{k=0}^{\lceil\log_2 T\rceil}
  2^k
  \mathbb P\!\left(
  n_{m+1}-n_m\in[2^{k-1},2^k)\mid n_m
  \right)}
  \right)
  +3\delta T .
\end{multline*}
We are now left with the task of bounding the length of each epoch, for which we rely on the following lemma.
\begin{lemma}
\label{lem:bound_n_m}
For every epoch \(m\), on the good event, if the active intervals can be covered by at most \(\rho\) active bid regions and \(F_S\) is \((\gamma,C_S)\)-regular on these regions, then there exists a universal constant \(C>0\) such that
\[
n_{m+1}-n_m
\le
C
\min\brc{
2^m\log\!\br{\frac{2T}{\delta}},
\,
\log^2\!\br{\frac{2T|\Gamma|}{\delta}}
\sqrt{\frac{\rho}{\gamma}}
+
\log\!\br{\frac{2T|\Gamma|}{\delta}}
}
2^{2m}
\eqqcolon S(m+1,T).
\]
If the optimistic regularity condition fails, the robust term still holds.
\end{lemma}
Two terms above correspond to the amount of steps needed to conclude an epoch by either the robust or the optimistic certificate. Moreover note that since this is deterministic quantity it holds that both 
\[
\lE[n_{m+1}-n_m]\leq S(m+1,T) \quad \text{and} \quad \lE\brs{ \sum_{k=0}^{\lceil\log_2 T\rceil} 2^k \mathbb{P}\!\left( n_{m+1}-n_m\in[2^{k-1},2^k)\mid n_m \right) } \leq S(m+1,T).
\]
Therefore, combining the regret of the algorithm and the bound on the discretization error or \cref{eq:disc_error}
we obtain that
\begin{align}
     \sum_{m=0}^{\lceil\log T\rceil-1}\lE\brs{ \sum_{t=n_m+1}^{n_{m+1}}  
 u_{v_t}(\hat{p}^*) - u_{v_t}(p_t)} + \lE\brs{\sum_{t=1}^T\frac{1}{| \cG_t|}}  \leq 2\sum_{m=0}^{\lceil\log T\rceil-1}2^{-m}S(m+1, T) + 3\delta T
\end{align}

The final regret bound is obtained by replacing the explicit value of $S(m+1, T)$. In the case this is equal to $C \log^2\!\br{ \frac{2T}{\delta} } \sqrt{\frac{\rho}{\gamma}}
         \, 2^{2m}$, let $M$ the biggest value for which $S(M+1, T) \leq T$, given by
\begin{align*}
 2^{2M}\log^2\!\br{ \frac{2T}{\delta} } \sqrt{\frac{\rho}{\gamma}} &\leq T\\
 M &\leq \frac{1}{2}\log\br{T\sqrt{\frac{\gamma}{\rho}}\frac{1}{\log^2(2T / \delta)}}
\end{align*}
then 
\begin{align}\label{eq:boun_1}
    2\sum_{m=0}^{\lceil\log T\rceil-1}2^{-m}S(m+1,T)
    &\le
    2\br{
    2^M C\log^2\!\br{\frac{2T|\Gamma|}{\delta}}
    \sqrt{\frac{\rho}{\gamma}}\,M
    +
    \frac{T}{2^M}
    } \notag\\
    &=
    \tilde{\cO}\br{\rho^{1/4}\sqrt T\,\gamma^{-1/4}}.
\end{align}
When instead $S(m+1, T) =  C2^{3m}\log(T / \delta)$, then $M=\frac{1}{3}\log(T/(C\log(T / \delta)))$, hence as before
\begin{align}\label{eq:boun_2}
    2\sum_{m=0}^{\lceil\log T\rceil-1}2^{-m}S(m+1, T)
\leq
    2\br{2^{2M}C\log( T / \delta)M  +  \frac{T}{2^M}}
=
    \tilde{D}T^{2 / 3}\log^{1 / 3}(T / \delta)
\end{align}
Therefore, combining \Cref{eq:boun_1}, and \cref{eq:boun_2}, and considering the initial bound of \cref{eq:regre_parts}
\begin{align*}
    R_T 
& \leq
    \frac{2}{| \cV|}T + \sum_{t=1}^T\lE\brs{\frac{1}{|\cG_t|}}+\lE\brs{
    u_{v_t}(\hat{p}^*) - u_{v_t}(p_t)} \\
&\leq
    \frac{2}{| \cV|}T + \min\brc{\tilde{\cO}\br{T^{2 / 3}}, \tilde{\cO}\br{\rho^{1/4}\sqrt T\,\gamma^{-1/4}}} +3\delta T
\end{align*}
choosing $\delta = 1 / T$ and $|\cV| = T$ yields the desired result.




\subsection{Proof of Auxiliary Lemmas}

\subsubsection{Proof of Lemma \ref{lem:reg_epoch}}
In order to prove this result, we need the following two lemmas that ensure that both estimates of the gaps nicely concentrate around the true value. 
\begin{lemma}
\label{lem:conc_delta_opt}
If \(F_S\) satisfies \cref{ass:local-shill-regularity}, then with probability at least \(1-\delta\), for every \(\bar\gamma\in\Gamma_m^{\adm}\), every \(v\in\cV\), every \(t\le T\), and every \(p,q\in A_m(v)\cap\cG_m\),
\[
\left|
\widehat\Delta_t^{\opt,\bar\gamma}(v,p,q)
-
\Delta(v,p,q)
\right|
\le
r_{\opt,\bar\gamma}(t,v).
\]
\end{lemma}
Likewise 
\begin{lemma}\label{lem:conc_delta_rob}
    Given $v\in \cV$, for all $p, q \in A_m(v)\cap\cG_m$ and $t\geq 1$ 
    \[|\hat{\Delta}_t^{rob}(v, p, q) - \Delta(v, p, q)|  \leq r_{\rob}(t, v)\,,\]
with probability at least $1-\delta$.
\end{lemma}
Note that conditioning on the events above with high probability, the optimal arm is not eliminated,
this is formalized by the following lemma.
\begin{lemma}
\label{lem:safe-elimination}
Fix an epoch \(m\), a value \(v\in\cV\), and a branch \(b\in\{\rob,\opt\}\). Let
\[
p_{m,G}^\star(v)
\in
\argmax_{p\in A_m(v)\cap\cG_m}u_B(v,p).
\]
Suppose that \(p_{m,G}^\star(v)\in A_m(v)\cap\cG_m\), and that for every \(p,q\in A_m(v)\cap\cG_m\),
\[
\left|
\widehat\Delta^b_t(v,q,p)-\Delta(v,q,p)
\right|
\le
r_b(t,v)
\le
c_{\rad}2^{-(m+1)}.
\]
If the elimination rule removes \(p\in A_m(v)\cap\cG_m\) only when there exists \(q\in A_m(v)\cap\cG_m\) such that
\[
\widehat\Delta^b_t(v,q,p)>c_{\el}2^{-(m+1)},
\qquad
c_{\el}>c_{\rad},
\]
then \(p_{m,G}^\star(v)\) is not eliminated. Moreover, every surviving grid bid \(p\) satisfies
\[
u_B(v,p_{m,G}^\star(v))-u_B(v,p)
\le
(c_{\el}+c_{\rad})2^{-(m+1)}.
\]
After taking the convex hull of the surviving grid bids, every bid in \(A_{m+1}(v)\) satisfies the same bound up to the epoch discretization error \(O(h_m)\).
\end{lemma}

Assume first that \(\iota_{m-1}=1\), meaning that \(A_m\) was constructed at the end of epoch \(m-1\) using the optimistic certificate with validating candidate \(\bar\gamma_{m-1}^\star\). Then, for every \(p\in A_m(v_t)\cap\cG_m\), on the optimistic concentration event,
\[
u_{v_t}(\hat p^\star_{v_t})-u_{v_t}(p)
=
\Delta(v_t,\hat p^\star_{v_t},p)
\le
\widehat\Delta^{\opt,\bar\gamma_{m-1}^\star}_{n_m}
(v_t,\hat p^\star_{v_t},p)
+
r_{\opt,\bar\gamma_{m-1}^\star}(n_m,v_t).
\]
Due to the elimination rule used at the end of epoch \(m-1\), and since
\(\iota_{m-1}=1\), every surviving grid point from
\(A_{m-1}(v)\cap\cG_{m-1}\) satisfies
\[
\widehat\Delta^{\opt,\bar\gamma_{m-1}^\star}_{n_m}
(v,q,p)
\le
2^{-m}
\qquad
\text{for all }q\in A_{m-1}(v)\cap\cG_{m-1}.
\]
By \(\Cref{lem:safe-elimination}\) and the convex-hull step, the same
\(2^{-m}\)-optimality guarantee, up to the epoch discretization error already
absorbed in the \(2^{-m}\) term, holds for every
\(p\in A_m(v)\cap\cG_m\).
Hence, being $\cE_1$ the event that the estimator $\hat{\Delta}$ concentrates around the average as expressed in \cref{lem:conc_delta_opt} and $\cE_2$ the event that the optimal arm is not eliminated then 
\begin{align*}
\lE\brs{u_{v_t}(\hat{p}^*_{v_t})-u_{v_t}(p) \mid \cH_{t-1}, v_t }
&= \lE\brs{u_{v_t}(\hat{p}^*_{v_t})-u_{v_t}(p) \mid \cH_{t-1}, v_t, \cE_1\cap \cE_2 }\mathbb{P}(\cE_1\cap \cE_2)\\
&+\lE\brs{u_{v_t}(\hat{p}^*_{v_t})-u_{v_t}(p) \mid \cH_{t-1}, v_t, \cE_1^c\cup \cE^c_2 }\mathbb{P}(\cE_1^c\cup \cE^c_2) \\
&\leq
2^{-m}
+
r_{\opt,\bar\gamma_{m-1}^\star}(n_m,v_t)
+
2\delta .
\end{align*}
By an analogous reasoning, it is possible to prove that if the elimination is done thanks to the robust estimator, then 
\begin{align*}
\lE\brs{u_{v_t}(\hat{p}^*_{v_t})-u_{v_t}(p) \mid \cH_{t-1}, v_t }
\leq
2^{-m}
+
r_{\rob}(n_m,v_t)
+
2\delta .
\end{align*}

From which, combining the two results, we obtain
\[
\lE\brs{
u_{v_t}(\hat p^\star_{v_t})-u_{v_t}(p)
\mid \cH_{t-1},v_t
}
\le
2^{-m}
+
\iota_{m-1} r_{\opt,\bar\gamma_{m-1}^\star}(n_m,v_t)
+
(1-\iota_{m-1})r_{\rob}(n_m,v_t)
+
2\delta .
\]

\subsubsection{Proof of Lemma \ref{lem:bound_avg_radius}}

By definition, during epoch \(m\),
\[
R(t,v)
=
\iota_{m-1} r_{\opt,\bar\gamma_{m-1}^\star}(n_m,v)
+
(1-\iota_{m-1})r_{\rob}(n_m,v),
\]
so the radius used in the epoch-regret bound is frozen at the beginning of epoch \(m\). Hence
\[
\sum_{t=n_m+1}^{n_{m+1}}R(t,v_t)
=
\sum_{t=n_m+1}^{n_{m+1}}R(n_m,v_t).
\]
Conditioning on \(n_m\), on \(\cH_{n_m}\), and on the certificate used to construct \(A_m\), and decomposing according to the dyadic value of \(n_{m+1}-n_m\), we get
\[
\begin{aligned}
\lE\brs{
\sum_{t=n_m+1}^{n_{m+1}} R(t,v_t)
}
&\le
\lE\brs{
\sum_{k=0}^{\lceil\log_2 T\rceil}
2^k
\lE\brs{R(n_m,V)\mid \cH_{n_m},\iota_{m-1}}
\mathbb P
\!\left(
n_{m+1}-n_m\in[2^{k-1},2^k)\mid n_m
\right)
}.
\end{aligned}
\]
By the stopping condition at the beginning of epoch \(m\), and by Hoeffding's inequality applied to the empirical average over values, with probability at least \(1-\delta\),
\(
\lE\brs{R(n_m,v_t)\mid n_m,\iota_{m-1}}\le 2^{-m}
\).
On the complementary event, we use the trivial bound \(R(t,v)\le 1\). This gives
\[
\lE\brs{
\sum_{t=n_m+1}^{n_{m+1}} R(t,v_t)
}
\le
2^{-m}
\lE\brs{
\sum_{k=0}^{\lceil\log_2 T\rceil}
2^k
\mathbb P
\!\left(
n_{m+1}-n_m\in[2^{k-1},2^k)\mid n_m
\right)
}
+
\delta\,
\lE[n_{m+1}-n_m].
\qedhere
\]

\subsubsection{Proof of Lemma~\ref{lem:bound_n_m}}
Epoch \(m\) ends when either the optimistic or the robust certificate reaches the threshold \(2^{-(m+1)}\).
For the robust branch, the total length of the epoch can be bounded by the following lemma 
\begin{lemma}
\label{lem:robust-epoch-length}
If the robust branch is run on grid \(\mathcal G_t\) during epoch \(m\), then after \( \tau_m \ge C|\mathcal G_t|2^{2m}\log\!\br{\frac{2T}{\delta}} \) rounds within epoch \(m\), its average radius is at most \(2^{-(m+1)}\).
\end{lemma}
Furthermore note that \(\cG_m\) is fixed during epoch \(m\) and \(|\cG_m|\asymp 2^m\). Hence, the robust certificate reaches the threshold after at most
\begin{align}\label{eq:lenght_rob}
C|\cG_m|2^{2m}\log\br{\frac{2T}{\delta}}
\le
C2^{3m}\log\br{\frac{2T}{\delta}}
\end{align}
rounds in epoch \(m\).\\
For the optimistic branch, let \(\bar\gamma\in\Gamma_m^{\adm}\) be the dyadic admissible candidate from \cref{lem:dyadic-gamma}, then the following lemma bounds the length of each epoch
\begin{lemma}[Optimistic epoch length]
\label{lem:optimistic-epoch-length}
Fix an epoch \(m\), and suppose that the optimistic branch is run on active intervals that can be covered by at most \(\rho\) active bid regions. Assume that \(F_S\) is \((\gamma,C_S)\)-regular on these regions, and let \(\bar\gamma\in\Gamma_m^{\adm}\) be the dyadic admissible candidate from \cref{lem:dyadic-gamma}. Then, on the optimistic good event, after
\[
\tau_m
\ge
C
\br{
\log^2\!\br{\frac{2T|\Gamma|}{\delta}}
\sqrt{\frac{\rho}{\gamma}}
+
\log\!\br{\frac{2T|\Gamma|}{\delta}}
}
2^{2m}
\]
optimistic rounds within epoch \(m\), the candidate \(\bar\gamma\) satisfies
\[
\widehat r_{\opt,\bar\gamma}(\tau_m)
+
\sqrt{\frac{\log(2T|\Gamma|/\delta)}{2\tau_m}}
\le
2^{-(m+1)}.
\]
Consequently, epoch \(m\) can end through the optimistic branch, and the elimination step is safe by \cref{lem:safe-elimination}.
\end{lemma}
Therefore during epoch $m$ the algorithm will run at most 
\begin{align}\label{eq:lenght_opt}
C
\br{
\log^2\!\br{\frac{2T|\Gamma|}{\delta}}
\sqrt{\frac{\rho}{\gamma}}
+
\log\!\br{\frac{2T|\Gamma|}{\delta}}
}
2^{2m}
\end{align}
optimistic rounds. Since the epoch stops as soon as either certificate validates, the epoch length is bounded by the minimum of the robust and optimistic displays, hence the minimum between \cref{eq:lenght_rob} and \cref{eq:lenght_opt}. If the optimistic regularity condition fails or no optimistic candidate validates, the robust display still applies.


\subsubsection{Proof of Lemma \ref{lem:conc_delta_opt}}

The optimistic estimator is obtained by combining two different kinds of observations, the direct ones and the global ones. For the former ones the following holds
\[
Y^{\mathrm{dir}}\coloneq W=\ind{q_i\ge B}.
\]
Then
\[
\mathbb E[Y^{\mathrm{dir}}\mid q_i]
=
F_B(q_i)
=
e_i^\top F.
\]
Thus a direct observation is a scalar linear measurement
\[
Y^{\mathrm{dir}}
=
\Phi^{\mathrm{dir}}F+\xi^{\mathrm{dir}},
\qquad
\Phi^{\mathrm{dir}}\coloneq e_i^\top,
\]
where
\[
\xi^{\mathrm{dir}}
=
W-F_B(q_i),
\qquad
\mathbb E[\xi^{\mathrm{dir}}\mid q_i]=0.
\]
Moreover,
\[
\operatorname{Var}(\xi^{\mathrm{dir}}\mid q_i)\le 1,
\qquad
|\xi^{\mathrm{dir}}|\le 1.
\]
Direct measurements are assigned weight
\[
\omega^{\mathrm{dir}}\coloneq 1.
\]
For the other type of estimator, instead, consider  an optimistic block on a certified active interval
\[
\cA^m(v)=\{q_{i_1}<\cdots<q_{i_k}\}.
\]
Here \(i_1,\ldots,i_k\) are global grid indices, while \(1,\ldots,k\) denote local positions inside the active interval. Suppose the algorithm plays the left endpoint \(q_{i_1}\). For every \(q_{i_j}\in \cA^m(v)\), define
\[
Y(q_{i_j})
=
W+(1-W)\frac{\ind{O\le q_{i_j}}}{F_S(q_{i_j})}.
\]
Since \(q_{i_1}\le q_{i_j}\), the suffix debiasing identity gives
\[
\mathbb E[Y(q_{i_j})\mid q_{i_1}]=F_B(q_{i_j}).
\]
Define
\[
a_j\coloneq v-q_{i_j},
\qquad j=1,\ldots,k.
\]
For adjacent active bids, the optimistic branch forms the suffix difference measurements
\[
Y^{\mathrm{suf}}_j
\coloneq 
a_{j+1}Y(q_{i_{j+1}})-a_jY(q_{i_j}),
\qquad
j=1,\ldots,k-1.
\]
Then
\[
\mathbb E[Y^{\mathrm{suf}}_j]
=
a_{j+1}F_B(q_{i_{j+1}})-a_jF_B(q_{i_j})
=u_v(i_{j+1})-u_v(i_j).
\]
Equivalently, the suffix vector \(Y^{\mathrm{suf}}\in\mathbb R^{k-1}\) satisfies
\[
Y^{\mathrm{suf}}
=
\Phi^{\mathrm{suf}}F+\xi^{\mathrm{suf}},
\qquad
\mathbb E[\xi^{\mathrm{suf}}\mid q_{i_1}]=0,
\]
where the \(j\)-th row of \(\Phi^{\mathrm{suf}}\) is
\[
\Phi^{\mathrm{suf}}_j
=
a_{j+1}e_{i_{j+1}}^\top-a_je_{i_j}^\top.
\]
On the active-set invariant event, the interval has maintained width $2^{-m}$.  Furthermore, it is possible to bound the variance of the estimator above as 
\begin{lemma}\label{lem:suffix-covariance}
    If \(F_S\) is locally regular on $\cA^m(v)$, during epoch $m$ when the active prices have gap smaller than $2^{-m}$ and while using a grid with width $h$, 
    \[
\operatorname{Cov}(\xi^{\mathrm{suf}}\mid q_{i_j})
\preceq
C_{\mathrm{suf}}
\frac{2^{-m}+h}{\gamma}I.
\]
\end{lemma} 

The oracle balanced suffix weight on epoch \(m\) is
\[
\omega_m^\star
=
\frac{\gamma}{C_{\suf}(2^{-m}+h_m)}.
\]
The algorithm does not know \(\gamma\). Instead, it checks the dyadic candidates \(\bar\gamma\in\Gamma\), using
\[
\omega_{m,\bar\gamma}
=
\frac{\bar\gamma}{C_{\suf}(2^{-m}+h_m)}.
\]
For the candidate from \cref{lem:dyadic-gamma}, this weight is within a universal constant factor of \(\omega_m^\star\). The concentration statement below is proved for each fixed \(\bar\gamma\in\Gamma_m^{\adm}\) and then union-bounded using \(|\Gamma_m^{\adm}|\le|\Gamma|\).

For a candidate \(\bar\gamma\in\Gamma_m^{\adm}\), every suffix row it uses lies on grid points satisfying \(F_S(q)\ge\bar\gamma\). Hence
\[
0\le Y(q_{i_j})\le \frac1{\bar\gamma},
\qquad
|Y_j^{\suf}|\le \frac2{\bar\gamma},
\qquad
\|\xi^{\suf}\|_\infty\le \frac4{\bar\gamma}.
\]
The final estimator is obtained by combining this information in the following way. At each round, the information collected is of the form
\[
Y_s=\Phi_sF+\xi_s,
\qquad
\mathbb E[\xi_s\mid\cH_{s-1}]=0.
\]

For each \(\bar\gamma\in\Gamma_m^{\adm}\), define
\[
G_t^{\bar\gamma}
=
G_t^{\dir}
+
\omega_{m,\bar\gamma}G_t^{\suf},
\qquad
z_t^{\bar\gamma}
=
z_t^{\dir}
+
\omega_{m,\bar\gamma}z_t^{\suf},
\qquad
\widehat F_t^{\opt,\bar\gamma}
=
(G_t^{\bar\gamma})^\dagger z_t^{\bar\gamma}.
\]
For a context \(v\) and active bids \(i,j\in A_m(v)\cap\cG_m\), define
\[
g_{v,i,j}
=
(v-q_i)e_i-(v-q_j)e_j,
\]
and
\[
\widehat\Delta_t^{\opt,\bar\gamma}(v,i,j)
=
g_{v,i,j}^\top \widehat F_t^{\opt,\bar\gamma}.
\]
If \(g_{v,i,j}\notin\operatorname{range}(G_t^{\bar\gamma})\), the corresponding radius is set to \(+\infty\). Otherwise define
\[
Q_t^{\bar\gamma}(v)
=
\max_{i,j\in A_m(v)\cap\cG_m}
g_{v,i,j}^\top
(G_t^{\bar\gamma})^\dagger
g_{v,i,j}.
\]
Define \(B_t^{\bar\gamma}(v)\) as the range certificate of the weighted design induced by \(\bar\gamma\). Equivalently, direct measurements are weighted by \(1\), suffix measurements are weighted by \(\omega_{m,\bar\gamma}\), and the deterministic noise bounds are \(1\) for direct measurements and \(4/\bar\gamma\) for suffix measurements.

 
Fix an epoch \(m\) and a candidate \(\bar\gamma\in\Gamma_m^{\adm}\). We prove concentration for the weighted least-squares certificate induced by this candidate. Direct measurements have weight \(1\), while suffix-difference measurements have weight
\[
\omega_{m,\bar\gamma}
=
\frac{\bar\gamma}{C_{\suf}(2^{-m}+h_m)}.
\]
For each measurement \(s\), define
\[
\alpha_s^{\bar\gamma}
=
\begin{cases}
1, & s\in\mathcal S_t^{\dir},\\[0.3em]
\omega_{m,\bar\gamma}, & s\in\mathcal S_t^{\suf},
\end{cases}
\]
where \(\mathcal S_t^{\dir}\) and \(\mathcal S_t^{\suf}\) denote respectively the direct and suffix-difference measurements collected up to time \(t\). With this notation,
\[
G_t^{\bar\gamma}
=
\sum_{s\le t}
\alpha_s^{\bar\gamma}\Phi_s^\top\Phi_s,
\qquad
z_t^{\bar\gamma}
=
\sum_{s\le t}
\alpha_s^{\bar\gamma}\Phi_s^\top Y_s,
\qquad
\widehat F_t^{\opt,\bar\gamma}
=
(G_t^{\bar\gamma})^\dagger z_t^{\bar\gamma}.
\]

We first verify the covariance normalization needed for Freedman's inequality. For a direct measurement, \(\alpha_s^{\bar\gamma}=1\), \(\Phi_s=\Phi_s^{\dir}\), and
\[
\operatorname{Var}(\xi_s\mid\cH_{s-1})\le 1.
\]
Therefore
\[
(\alpha_s^{\bar\gamma})^2
\Phi_s^\top
\operatorname{Cov}(\xi_s\mid\cH_{s-1})
\Phi_s
\preceq
\Phi_s^\top\Phi_s
=
\alpha_s^{\bar\gamma}\Phi_s^\top\Phi_s.
\]

Consider now a suffix-difference measurement. Since \(\bar\gamma\in\Gamma_m^{\adm}\), every grid point involved in the suffix row satisfies
\[
F_S(q)\ge \bar\gamma,
\]
and every adjacent pair involved in the active interval satisfies
\[
F_S(q_{j+1})-F_S(q_j)
\le C_{\val}\bar\gamma h_m.
\]
Thus the suffix covariance bound applies with \(\bar\gamma\) in place of the local shill scale, after increasing constants if necessary:
\[
\operatorname{Cov}(\xi_s\mid\cH_{s-1})
\preceq
C_{\suf}\frac{2^{-m}+h_m}{\bar\gamma}I.
\]
Since
\[
\alpha_s^{\bar\gamma}
=
\omega_{m,\bar\gamma}
=
\frac{\bar\gamma}{C_{\suf}(2^{-m}+h_m)},
\]
we get
\[
\begin{aligned}
(\alpha_s^{\bar\gamma})^2
\Phi_s^\top
\operatorname{Cov}(\xi_s\mid\cH_{s-1})
\Phi_s
&\preceq
(\alpha_s^{\bar\gamma})^2
C_{\suf}\frac{2^{-m}+h_m}{\bar\gamma}
\Phi_s^\top\Phi_s \\
&=
\alpha_s^{\bar\gamma}\Phi_s^\top\Phi_s.
\end{aligned}
\]
Hence, for every measurement \(s\),
\[
(\alpha_s^{\bar\gamma})^2
\Phi_s^\top
\operatorname{Cov}(\xi_s\mid\cH_{s-1})
\Phi_s
\preceq
\alpha_s^{\bar\gamma}\Phi_s^\top\Phi_s.
\]

We now prove the concentration inequality. Fix \(t\), \(v\in\cV\), and
\(q_i,q_j\in A_m(v)\cap\cG_m\). Let
\[
g\coloneq g_{v,i,j}=(v-q_i)e_{q_i}-(v-q_j)e_{q_j}.
\]
If \(g\notin\operatorname{range}(G_t^{\bar\gamma})\), then the algorithm sets
\(r_{\opt,\bar\gamma}(t,v)=+\infty\), and the claim is trivial. We therefore assume
\(g\in\operatorname{range}(G_t^{\bar\gamma})\).

Since each measurement satisfies
\[
Y_s=\Phi_sF+\xi_s,
\qquad
\mathbb E[\xi_s\mid\cH_{s-1}]=0,
\]
we have
\[
z_t^{\bar\gamma}
=
\sum_{s\le t}\alpha_s^{\bar\gamma}\Phi_s^\top Y_s
=
G_t^{\bar\gamma}F
+
\sum_{s\le t}\alpha_s^{\bar\gamma}\Phi_s^\top\xi_s.
\]
Thus
\[
\widehat F_t^{\opt,\bar\gamma}
=
(G_t^{\bar\gamma})^\dagger G_t^{\bar\gamma}F
+
(G_t^{\bar\gamma})^\dagger
\sum_{s\le t}\alpha_s^{\bar\gamma}\Phi_s^\top\xi_s.
\]
Because \(g\in\operatorname{range}(G_t^{\bar\gamma})\),
\[
g^\top (G_t^{\bar\gamma})^\dagger G_t^{\bar\gamma}F
=
g^\top F.
\]
Therefore
\[
\widehat\Delta_t^{\opt,\bar\gamma}(v,q_i,q_j)-\Delta(v,q_i,q_j)
=
g^\top (G_t^{\bar\gamma})^\dagger
\sum_{s\le t}\alpha_s^{\bar\gamma}\Phi_s^\top\xi_s.
\]
For fixed \(t\), define
\[
X_{s,t}^{\bar\gamma}
\coloneq 
\alpha_s^{\bar\gamma}
g^\top (G_t^{\bar\gamma})^\dagger\Phi_s^\top\xi_s,
\qquad s\le t.
\]
Conditional on the outcome-oblivious design, \(G_t^{\bar\gamma}\), \(\Phi_s\), and
\(\alpha_s^{\bar\gamma}\) are fixed. Hence
\((X_{s,t}^{\bar\gamma})_{s\le t}\) is a martingale difference sequence. Its predictable quadratic variation satisfies
\[
\begin{aligned}
\sum_{s\le t}
\mathbb E\!\left[(X_{s,t}^{\bar\gamma})^2\mid\cH_{s-1}\right]
&=
\sum_{s\le t}
(\alpha_s^{\bar\gamma})^2
g^\top (G_t^{\bar\gamma})^\dagger
\Phi_s^\top
\operatorname{Cov}(\xi_s\mid\cH_{s-1})
\Phi_s
(G_t^{\bar\gamma})^\dagger g \\
&\le
\sum_{s\le t}
\alpha_s^{\bar\gamma}
g^\top (G_t^{\bar\gamma})^\dagger
\Phi_s^\top\Phi_s
(G_t^{\bar\gamma})^\dagger g \\
&=
g^\top (G_t^{\bar\gamma})^\dagger
G_t^{\bar\gamma}
(G_t^{\bar\gamma})^\dagger g \\
&=
g^\top (G_t^{\bar\gamma})^\dagger g \\
&\le
Q_t^{\bar\gamma}(v).
\end{aligned}
\]
The third equality uses
\[
G_t^{\bar\gamma}
=
\sum_{s\le t}\alpha_s^{\bar\gamma}\Phi_s^\top\Phi_s,
\]
and the fourth uses \(g\in\operatorname{range}(G_t^{\bar\gamma})\).

It remains to control the increments. For direct measurements, \(|\xi_s|\le 1\). For suffix measurements, admissibility gives \(F_S(q)\ge\bar\gamma\) on every involved grid point, so
\[
0\le Y(q)\le \frac1{\bar\gamma},
\qquad
|Y_s^{\suf}|\le \frac2{\bar\gamma},
\qquad
|\xi_s^{\suf}|\le \frac4{\bar\gamma}.
\]
Thus the deterministic noise bound is
\[
R_s^{\bar\gamma}
=
\begin{cases}
1, & s\in\mathcal S_t^{\dir},\\[0.3em]
4/\bar\gamma, & s\in\mathcal S_t^{\suf}.
\end{cases}
\]
By the definition of the range certificate in \cref{eq:range-cert-gamma},
\[
|X_{s,t}^{\bar\gamma}|
\le
B_t^{\bar\gamma}(v)
\qquad
\text{for every }s\le t.
\]
Freedman's inequality gives, for this fixed \(t,v,q_i,q_j,\bar\gamma\), with probability at least \(1-\eta\),
\[
\left|
\widehat\Delta_t^{\opt,\bar\gamma}(v,q_i,q_j)-\Delta(v,q_i,q_j)
\right|
\le
\sqrt{2\log(1/\eta)Q_t^{\bar\gamma}(v)}
+
\frac23\log(1/\eta)B_t^{\bar\gamma}(v).
\]
Taking a union bound over times, values, active grid pairs, and candidates in
\(\Gamma_m^{\adm}\), and using \(|\Gamma_m^{\adm}|\le|\Gamma|\), gives
\[
\left|
\widehat\Delta_t^{\opt,\bar\gamma}(v,q_i,q_j)-\Delta(v,q_i,q_j)
\right|
\le
\sqrt{
2\log\!\br{\frac{2T|\Gamma|}{\delta}}
Q_t^{\bar\gamma}(v)
}
+
\frac23
\log\!\br{\frac{2T|\Gamma|}{\delta}}
B_t^{\bar\gamma}(v).
\]
The right-hand side is \(r_{\opt,\bar\gamma}(t,v)\), which proves the concentration claim. Furthermore the quantity $Q_t^{\bar\gamma}(v)$ can be bound as follows.

\begin{lemma}
\label{lem:optimistic-path-bound}
Let
\[
Q_t^{\bar\gamma}(v)
\coloneq 
\max_{q_i,q_j\in A_m(v)\cap\cG_m}
g_{v,i,j}^{\top}(G_t^{\bar\gamma})^\dagger g_{v,i,j}.
\]
For \(q\in A_m(v)\cap\cG_m\), let \(W_t(q)\) be the number of direct measurements at \(q\), and let
\(
M_t^{\bar\gamma}(q)
\coloneq 
\omega_{m,\bar\gamma}N_t^{\suf}(q)
\)
be the weighted suffix mass at \(q\). Then
\begin{multline*}
Q_t^{\bar\gamma}(v)
\le
C\min\Bigg\{
    \frac{1}{\min_{q\in A_m(v)\cap\cG_m}W_t(q)},
    \frac{|A_m(v)\cap\cG_m|}
         {\min_{q\in A_m(v)\cap\cG_m}M_t^{\bar\gamma}(q)},
\\ \quad
    \sqrt{
        \frac{1}{
        \min_{q\in A_m(v)\cap\cG_m}W_t(q)
        \min_{q\in A_m(v)\cap\cG_m}M_t^{\bar\gamma}(q)
        }
    }
\Bigg\}.
\end{multline*}
\end{lemma}

\subsubsection{Proof of Lemma \ref{lem:conc_delta_rob}}
For every fixed value $v$ and couple $p, q \in A_m(v)\cap\cG_m$, thanks to a direct application of the Hoeffding-Azuma Inequality \citep{boucheron2003concentration}, in the conditional version 
it holds that with probability at least $1-\delta$
\[
| \hat{\Delta}_t^{rob}(v, p, q) - \Delta(v, p, q)| \leq \max_{k \in \{p,q \}}\sqrt{\frac{2\log (2T / \delta)}{N_t(k)}}
\]
therefore with the same probability, it holds that
\[
| \hat{\Delta}_t^{rob}(v, p, q) - \Delta(v, p, q)| \leq \max_{q \in \cA(v)}\sqrt{\frac{2\log (2T / \delta)}{N_t(q)}}
\]

\subsubsection{Proof of Lemma \ref{lem:safe-elimination}}
For any \(q\in A_m(v)\cap\cG_m\), optimality of \(p_{m,G}^\star(v)\) gives
\(\Delta(v,q,p_{m,G}^\star(v))\le 0\). Hence, on the concentration event,
\[
\widehat\Delta^b_t(v,q,p_{m,G}^\star(v))
\le r_b(t,v)
\le c_{\mathrm{rad}}2^{-(m+1)}
<
c_{\mathrm{el}}2^{-(m+1)} ,
\]
so \(p_{m,G}^\star(v)\) cannot be eliminated.
Now let \(p\) survive the elimination step. Taking \(q=p_{m,G}^\star(v)\), survival implies
\(\widehat\Delta^b_t(v,p_{m,G}^\star(v),p)\le c_{\mathrm{el}}2^{-(m+1)}\). Therefore,
\[
\Delta(v,p_{m,G}^\star(v),p)
\le
\widehat\Delta^b_t(v,p_{m,G}^\star(v),p)+r_b(t,v)
\le
(c_{\mathrm{el}}+c_{\mathrm{rad}})2^{-(m+1)} .
\]
Finally, taking the convex hull preserves this bound by \cref{lem:rhr-implies-interval-level-sets}.

\subsubsection{Proof of Lemma \ref{lem:robust-epoch-length}}
Writing explicitly the expression of $\hat{r}_{\rob}(\tau_m)$, we have that $n_m$ corresponds to the first instant of time when
\[
\frac{1}{\tau_m}\sum_{s=1}^{\tau_m} \max_{q\in \cA(v_s)}{\sqrt{\frac{2\log(2T / \delta)}{N_{\tau_m}(q)}}}+\sqrt{\frac{\log(2T / \delta)}{2\tau_m}}\leq2^{-(m+1)} 
\]
Since the robust branch is playing the arm in a round robin fashion during the odd rounds and since the size of the grid is fixed during an epoch, it holds that $\forall q, N_{\tau_m}(q) \geq \tau_m / (2| \cG_{m}|) -1$, hence
\[
\frac{1}{\tau_m} \sum_{s=1}^{\tau_m} \max_{q\in \cA(v_s)}{\sqrt{\frac{2\log(2T / \delta)}{N_{\tau_m}(q)}}}
+ \sqrt{\frac{\log(2T / \delta)}{2 \tau_m}}
\leq {\sqrt{\frac{2\log(2T / \delta)}{\tau_m / (2| \cG_{m}|) - 1}}}
+ \sqrt{\frac{\log(2T / \delta)}{2 \tau_m}}
\]
and since the former term is dominant it holds that, the inequality sums up to 
\[
 2\sqrt{\frac{2\log(2T / \delta)}{\tau_m / (2| \cG_{m}|) - 1}} \leq 2^{-(m+1)}
\]
which gives 
\[
\tau_m \geq \frac{64\vert \cG_m\vert \log(2T / \delta)}{2^{-2m}}\,.
\qedhere
\]

\subsubsection{Proof of Lemma \ref{lem:optimistic-epoch-length}}

By \cref{lem:dyadic-gamma}, the candidate \(\bar\gamma\) has weight within a universal constant factor of the oracle balanced weight
\[
\omega_m^\star
=
\frac{\gamma}{C_{\suf}(2^{-m}+h_m)}.
\]
Thus the covariance normalization in \cref{lem:suffix-covariance} and the path bound in \cref{lem:optimistic-path-bound} apply to the weighted certificate induced by \(G_t^{\bar\gamma}\), after changing constants.

For every value \(v\), \cref{lem:optimistic-path-bound} gives
\[
Q_t^{\bar\gamma}(v)
\le
C\frac{1}{\sqrt{W_{\min}(v)M_{\min}(v)}}.
\]
Under the uniform optimistic sampling rule over \(A_m(v)\cap\cG_m\), the active bid regions receive their proportional share of optimistic samples. Since the active intervals are covered by at most \(\rho\) bid regions, this gives
\[
\frac{1}{\tau_m}\sum_{s=n_m}^{t} Q_t^{\bar\gamma}(v_s)
\le
C\frac{\sqrt{\rho/\gamma}}{\tau_m}.
\]
Moreover, the range bound gives
\[
B_t^{\bar\gamma}(v)\le C\sqrt{Q_t^{\bar\gamma}(v)}.
\]
Therefore,
\[
r_{\opt,\bar\gamma}(t,v)
=
\sqrt{
2\log\!\br{\frac{2T|\Gamma|}{\delta}}Q_t^{\bar\gamma}(v)
}
+
\frac23
\log\!\br{\frac{2T|\Gamma|}{\delta}}B_t^{\bar\gamma}(v)
\le
C\log\!\br{\frac{2T|\Gamma|}{\delta}}
\sqrt{Q_t^{\bar\gamma}(v)}.
\]
By Jensen's inequality,
\[
\widehat r_{\opt,\bar\gamma}(\tau_m)
=
\frac1{\tau_m}\sum_{s=n_m}^{t}
r_{\opt,\bar\gamma}(t,v_s)
\le
C
\log\!\br{\frac{2T|\Gamma|}{\delta}}
\sqrt{
\frac1{\tau_m}\sum_{s=n_m}^{t}Q_t^{\bar\gamma}(v_s)
}
\le
C
\log\!\br{\frac{2T|\Gamma|}{\delta}}
\sqrt{\frac{\sqrt{\rho/\gamma}}{\tau_m}}.
\]
Thus, if
\[
\tau_m
\ge
C
\log^2\!\br{\frac{2T|\Gamma|}{\delta}}
\sqrt{\frac{\rho}{\gamma}}
2^{2m},
\]
then
\[
\widehat r_{\opt,\bar\gamma}(\tau_m)
\le
2^{-(m+2)}.
\]
Similarly, if
\[
\tau_m
\ge
C
\log\!\br{\frac{2T|\Gamma|}{\delta}}
2^{2m},
\]
then
\[
\sqrt{\frac{\log(2T|\Gamma|/\delta)}{2\tau_m}}
\le
2^{-(m+2)}.
\]
Combining the two bounds gives the claim.

\subsubsection{Proof of Lemma \ref{lem:suffix-covariance}}
In order to prove the statement of this lemma, we consider a more generic statement showing that if $F_S$ satisfies \cref{ass:local-shill-regularity} on the active interval, then for every $x\in\R^{k-1}$,
\begin{equation}
  \operatorname{Var}\br{ \sum_{j=1}^{k-1}x_jY_j^{suf} }
  \le
  C\frac{w+h}{\gamma}\|x\|_2^2.\label{eq:suffix-cov-bound}
\end{equation}
Equivalently, $\operatorname{Cov}(Y)\preceq C(w+h)\gamma^{-1}I$.

Let $X_j=Z_{j+1}-Z_j$. Since $a_j=a_{j+1}+h$,
\begin{equation*}
  Y_j^{suf}=a_{j+1}X_j-hY_j,
  \qquad
  Y_j=Y_1+\sum_{r<j}X_r.
\end{equation*}
Thus
\begin{equation}
  \sum_jx_jY_j^{suf}
  =-h\br{ \sum_jx_j }Y_1+
  \sum_r\beta_rX_r,
  \qquad
  \beta_r=a_{r+1}x_r-h\sum_{j>r}x_j.\label{eq:beta-decomposition}
\end{equation}
We bound the two terms in \cref{eq:beta-decomposition} separately, using $(A+B)^2\le2A^2+2B^2$.

For the $Y_1$ term, $\lE[Y_1^2]\lesssim\gamma^{-1}$ because the value $1/F_S(q_1)$ occurs with probability at most $F_S(q_1)$. Therefore
\begin{equation}
  h^2\br{ \sum_jx_j }^2\lE[Y_1^2]
  \le
  Ch^2k\gamma^{-1}\|x\|_2^2
  \le
  C\frac{h}{\gamma}\|x\|_2^2.\label{eq:Z1-term-bound}
\end{equation}

It remains to bound the latter term. Write $s_j=F_S(q_j)$. On $B\le q_1$, all $X_r$ vanish. On $B>q_1$, decompose $X_r=U_r+V_r$ with
\begin{align*}
  U_r
  &=\ind{q_r<B\le q_{r+1}}\frac{\ind{S\le q_{r+1}}}{s_{r+1}},\\
  V_r
  &=\ind{q_1<B\le q_r}
  \br{
  \frac{\ind{S\le q_{r+1}}}{s_{r+1}}
  -
  \frac{\ind{S\le q_r}}{s_r}
  }.
\end{align*}
For the $U$ part, the $B$-cells are disjoint and $\lE[U_r^2]\le\Delta_r/\gamma$. Hence, using the fact that by definition $a_{r+1}\Delta_r\le w+h$,
\begin{equation}
    \lE\brs{\br{ \sum_r\beta_rU_r }^2}
  \le \gamma^{-1}\sum_r\beta_r^2\Delta_r 
  \le
  C\frac{w+h}{\gamma}\|x\|_2^2.\label{eq:U-bound}
\end{equation}
The last step uses $\beta_r^2\le2a_{r+1}^2x_r^2+2h^2(\sum_{j>r}x_j)^2$, $a_{r+1}\Delta_r\le w+h$, $a_{r+1}\le1$, and $h^2k\le h$.

For the $V$ part, define
\begin{equation*}
  M_r=\frac{\ind{S\le q_r}}{s_r},
  \qquad
  R_r=M_{r+1}-M_r.
\end{equation*}
Then $V_r=\ind{q_1<B\le q_r}R_r$. The increments $R_r$ are orthogonal for distinct $r$, and local regularity gives
\begin{equation*}
  \lE[R_r^2]=\frac1{s_r}-\frac1{s_{r+1}}\le C_S\frac h\gamma.
\end{equation*}
Using independence of $B$ and $S$,
\begin{align}
  \lE\brs{\br{ \sum_r\beta_rV_r }^2}
  &\le C\frac h\gamma\sum_r\beta_r^2
  \le C\frac h\gamma\|x\|_2^2.\label{eq:V-bound}
\end{align}
Combining \cref{eq:Z1-term-bound,eq:U-bound,eq:V-bound} proves \cref{eq:suffix-cov-bound}.

\subsubsection{Proof of Lemma \ref{lem:optimistic-path-bound}}

Fix \(\bar\gamma\in\Gamma_m^{\adm}\). Throughout this proof write
\(G_t=G_t^{\bar\gamma}\) and \(M_t(q)=M_t^{\bar\gamma}(q)\).
Fix a value \(v\in\cV\), and write
\(A_m(v)\cap\cG_m=\{q_1<\cdots<q_k\}\). Let
\(a_j=v-q_j\), and let \(A_v=\mathrm{diag}(a_1,\ldots,a_k)\). We decompose
\(G_t=G_{\dir}+G_{\suf}\), corresponding respectively to direct and suffix measurements.
For the direct part,
\[
G_{\dir}\succeq
\min_{q_j\in A_m(v)} W_t(q_j) I.
\]
For the suffix part, the rows of the suffix design matrix are
\(a_{j+1}e_{j+1}^\top-a_je_j^\top\). Hence, if \(D\) is the unscaled first-difference matrix on the interval,
\[
G_{\suf}
\succeq
\min_{q_j\in A_m(v)} M_t(q_j)\, A_vD^\top D A_v .
\]
Therefore,
\[
G_t
\succeq
W_{\min}\br{ I+\beta A_vD^\top D A_v },
\qquad
\beta\coloneq \frac{M_{\min}}{W_{\min}},
\]
where
\[
W_{\min}\coloneq \min_{q_j\in A_m(v)}W_t(q_j),
\qquad
M_{\min}\coloneq \min_{q_j\in A_m(v)}M_t(q_j).
\]

Now fix \(i,j\in A_m(v)\), and let \(d=e_i-e_j\). Since
\(g_{v,i,j}=A_vd\), PSD domination gives
\[
g_{v,i,j}^\top G_t^\dagger g_{v,i,j}
\le
\frac1{W_{\min}}
d^\top A_v\br{ I+\beta A_vD^\top D A_v }^\dagger A_vd .
\]
By the variational form of the inverse quadratic form, and using \(0\le a_j\le 1\), the last display is bounded by
\[
\frac1{W_{\min}}
d^\top \br{ I+\beta D^\top D }^\dagger d .
\]

To bound the above quantity we consider the following result
\begin{lemma}[Path inverse bound]
\label{lem:path-inverse-bound}
Let \(D\) be the unscaled first-difference matrix on an interval of \(k\) grid points, and let
\(L_{\mathrm{path}}=D^\top D\). For \(\beta>0\), set $A_\beta=I+\beta L_{\mathrm{path}}$.
For \(a,b\in[k]\), let \(d=e_a-e_b\) and \(r=|a-b|\). Then
\[
d^\top A_\beta^{-1}d
\le
C\min\brc{
1,\frac r\beta,\frac1{\sqrt\beta}
}.
\]
\end{lemma}
Applying this with \(r=|i-j|\), we obtain
\[
g_{v,i,j}^\top G_t^\dagger g_{v,i,j}
\le
\frac{C}{W_{\min}}
\min\brc{
1,\frac{r}{\beta},\frac1{\sqrt\beta}
}.
\]
Substituting \(\beta=M_{\min}/W_{\min}\) gives
\[
g_{v,i,j}^\top G_t^\dagger g_{v,i,j}
\le
C\min\brc{
\frac1{W_{\min}},
\frac{r}{M_{\min}},
\frac1{\sqrt{W_{\min}M_{\min}}}
}.
\]
Taking the maximum over \(i,j\in A_m(v)\), and using \(r\le |A_m(v)|\), proves the claim.

\subsubsection{Proof of Lemma \ref{lem:path-inverse-bound}}



We prove the three bounds separately.
\begin{enumerate}
    \item  Since \(A_\beta\succeq I\), we have
    \(A_\beta^{-1}\preceq I\), and therefore
    \[
    d^\top A_\beta^{-1}d\le \|d\|_2^2\le 2.
    \]
    This gives the first bound, up to constants.

    \item Next, \(d\) is orthogonal to the constant vector. On the orthogonal complement of constants,
    \(L_{\mathrm{path}}\) is positive definite and \(A_\beta\succeq \beta L_{\mathrm{path}}\). Hence
    \[
    d^\top A_\beta^{-1}d
    \le
    \frac1\beta d^\top L_{\mathrm{path}}^\dagger d.
    \]
    For a unit-resistance path,
    \[
    d^\top L_{\mathrm{path}}^\dagger d=|a-b|=r,
    \]
    which gives the bound \(r/\beta\).

    \item It remains to prove the Green's-function bound. Let
    \(0=\lambda_0<\lambda_1\le\cdots\le\lambda_{k-1}\) be the eigenvalues of
    \(L_{\mathrm{path}}\), with orthonormal eigenvectors \(\phi_m\). The constant mode does not contribute because \(d\perp \mathbf 1\). For the path Laplacian,
    \(\lambda_m\ge c m^2/k^2\), and \(\|\phi_m\|_\infty^2\le C/k\). Thus
    \[
    d^\top A_\beta^{-1}d
    =
    \sum_{m=1}^{k-1}
    \frac{\langle d,\phi_m\rangle^2}{1+\beta\lambda_m}
    \le
    \frac Ck
    \sum_{m=1}^{k-1}
    \frac1{1+c\beta m^2/k^2}.
    \]
    Bounding the sum by an integral gives
    \[
    \frac Ck
    \sum_{m=1}^{k-1}
    \frac1{1+c\beta m^2/k^2}
    \le
    C\int_0^1\frac{dx}{1+c\beta x^2}+\frac Ck
    \le
    \frac C{\sqrt\beta},
    \]
    after increasing \(C\) to cover \(\beta\le 1\).
\end{enumerate}
Combining the three estimates proves the lemma.

\section{Proofs of Section \ref{sec:lower-bound}}
\label{app:proofs-lower-bound}

\subsection{Proof of Theorem \ref{th:single-region-lower-bound}}
\label{app:proof-of-single-region-lower-bound}

\singleregionlowerbound*

\begin{proof}
The goal is to construct a hard instance \(\nu\). Start by taking the value to be constant \(v_t\equiv 1\) and thus the utility of bid \(p\) is \(u(p)=(1-p)F_B(p)\).
Let the hard interval be
\[
J=\left[\frac13,\frac12\right].
\]
We first define a baseline buyer CDF \(F_{\base}\). On \(J\), set
\[
F_{\base}(p)=\frac{1}{4(1-p)}.
\]
Then
\[
(1-p)F_{\base}(p)=\frac14
\qquad
\text{for every }p\in J.
\]
So every bid in \(J\) is optimal under the baseline. To make the construction a valid CDF on the whole bid space, define
\[
F_{\base}(p)=
\begin{cases}
\frac98 p,          & 0 \le p \le \frac13, \\
\frac{1}{4(1-p)},   & \frac13 \le p \le \frac12, \\
\frac12,            & \frac12 \le p < 1, \\
1                   & p = 1.
\end{cases}
\]
This is a nondecreasing CDF and the related utility \(u_\base(p)=(1-p)F_{\base}(p)\) is weakly increasing on \([0,1/3]\), flat on \(J\), and weakly decreasing on \([1/2,1]\), hence \(u_\base\) has interval upper level sets.

We now plant a hidden cell. Let \(0<h\le |J|/4\) be a cell width and let
\[
N=\left\lfloor \frac{|J|}{h}\right\rfloor .
\]
Choose \(N\) disjoint cells \(I_1,\ldots,I_N\) of width \(h\) inside \(J\). For each cell \(I_a = [l_a, r_a]\), choose a unimodal bump function \(b_a: [0, 1] \to [0, 1]\) which is zero outside \(I_a\) and smoothly peaks at \(1\) in the middle of \(I_a\), formally we pick
\[
b_a(p) =
\begin{cases}
    \sin^2\br{\pi \frac{p - l_a}{h}} & p \in I_a, \\
    0 & p \notin I_a.
\end{cases}
\]
For this choice, there is an absolute constant \(C_b>0\), depending only on the chosen bump shape, such that \(\|b_a'\|_\infty\le C_b/h\). For the sine-squared bump above, \(C_b\) is a fixed numerical constant.

For a gap parameter \(\varepsilon>0\), define the planted CDF, for \(p<1\), by
\[
F_a(p)=F_{\base}(p)+\frac{\varepsilon}{1-p}b_a(p),
\]
and set \(F_a(1)=1\). On \(J\), the corresponding utility is
\[
u_a(p)=(1-p)F_a(p)=\frac14+\varepsilon b_a(p).
\]
Thus the planted cell contains the unique profitable region.
It is useful to define this profitable region directly from the utility gap. Let
\[
G_a:=\left\{p\in J:u_a^*-u_a(p)\le \frac{\varepsilon}{2}\right\}.
\]
Since \(u_a^*=1/4+\varepsilon\) and \(u_a(p)=1/4+\varepsilon b_a(p)\) on \(J\), this is
\[
G_a
= \left\{p\in J:b_a(p)\ge \frac12\right\}
= \brs{ l_a + \frac{h}{4}, l_a + \frac{3h}{4} }.
\]
Because \(b_a\) is unimodal, \(G_a\) is an interval contained in \(I_a\), and \(|G_a|\asymp h\). Every bid outside \(G_a\) incurs regret at least \(\varepsilon/2\):
\[
u_a^*-u_a(p)\ge \frac{\varepsilon}{2}
\qquad
\text{for every }p\notin G_a.
\]
Moreover, \(u_a\) is a flat plateau plus one unimodal bump, so its upper level sets are intervals.

\begin{figure}[t]
        \centering
        \includegraphics[width=\textwidth, page=2]{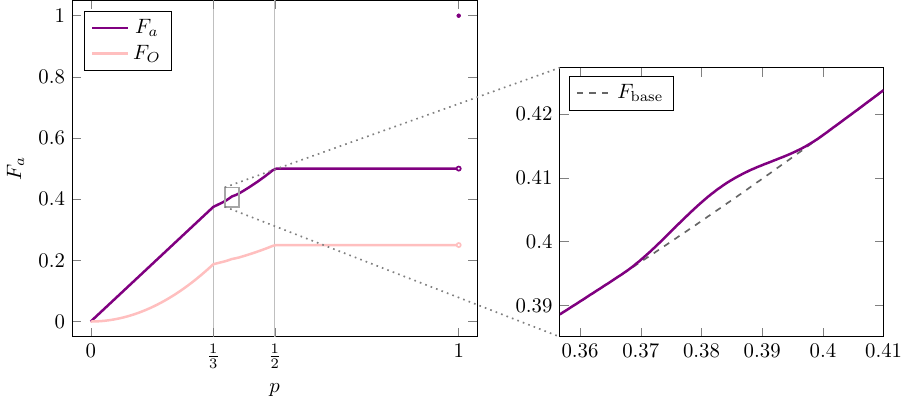}
        \caption{Planted buyer CDF $F_a$ and shilled CDF $F_O = F_a F_S$ with \(N = 5\) and \(a = 2\) in the interval \(J = [1/3, 1/2]\). The baseline distribution is unchanged outside the hard interval \(J\), while inside the planted cell \(I_a\subset J\) a localized bump raises the utility of bids in that cell. The perturbation vanishes at the endpoints of \(I_a\), so the planted CDF matches the baseline on neighboring cells. After \(J\), the CDF is flat, with the remaining probability mass placed as a common atom at \(1\). The shilled CDF is a scaled down version of the planted CDF, which further hides the bump.}
    \label{fig:lower-bound-cdf}
\end{figure}

We also need \(F_a\) to be a valid CDF. On the support of \(b_a\),
\[
\frac{d}{dp}\left(\frac{\varepsilon}{1-p}b_a(p)\right)
=
\frac{\varepsilon}{1-p}b_a'(p)
+
\frac{\varepsilon}{(1-p)^2}b_a(p).
\]
Since \(p\in J=[1/3,1/2]\), we have \((1-p)^{-1}\le 2\) and \((1-p)^{-2}\le 4\). Hence, using \(h\le 1\),
\[
\left|
\frac{d}{dp}\left(\frac{\varepsilon}{1-p}b_a(p)\right)
\right|
\le
2(C_b+2)\frac{\varepsilon}{h}.
\]
On \(J\), the baseline derivative satisfies
\[
F_{\base}'(p)=\frac{1}{4(1-p)^2}\ge \frac{9}{16}.
\]
Therefore, if
\[
\varepsilon\le \frac{9}{64(C_b+2)}h,
\]
then \(F_a'\ge0\) on \(J\). Outside \(J\), the perturbation is zero and \(F_a=F_{\base}\). Also, \(F_{\base}\le1/2\) on \([0,1)\), and the perturbation has size at most \(2\varepsilon\), so taking \(\varepsilon\le1/8\) ensures \(F_a\le3/4\) on \([0,1)\). Thus \(F_a\) is a valid CDF.

We now define the shill distribution. Put mass \(\gamma\) on a low interval below \(J\), say inside \([0,1/8]\), and put the remaining mass \(1-\gamma\) at the point \(1\). Equivalently,
\[
S=
\begin{cases}
S^{\mathrm{low}}, & \text{with probability }\gamma,\\
1, & \text{with probability }1-\gamma,
\end{cases}
\]
where \(S^{\mathrm{low}}\) is supported below \(J\). Then, for every bid \(q\in J\),
\[
F_S(q)=\gamma.
\]
The low-shill branch is the informative branch: when it occurs, the losing-side report preserves information about the buyer bid in the hard interval. On the high-shill branch, the losing-side report is always \(1\), so it carries no buyer-side information beyond the win-loss event. See \cref{fig:lower-bound-cdf}.

Fix an arbitrary policy \(\rho\). For each planted cell \(a\in[N]\), let \(P_a^\rho\) denote the law of the full transcript generated by policy \(\rho\) when the planted bump is in cell \(a\). Let \(P_\base^\rho\) denote the corresponding transcript law under the flat baseline distribution \(F_{\base}\). We compare each planted law to the baseline law in the reverse direction. By the chain rule for KL divergence over the adaptive transcript,
\[
\mathrm{KL}\!\left(P_\base^\rho\,\middle\|\,P_a^\rho\right)
=
\sum_{t=1}^T
\mathbb E_\base^\rho\!\left[
\mathrm{KL}\!\left(
Q_{0,p_t}
\,\middle\|\,
Q_{a,p_t}
\right)
\right],
\]
where \(p_t=\rho_t(H_{t-1})\), and \(Q_{a,p}\) denotes the one-round feedback law at bid \(p\) under instance \(a\). The action \(p_t\) itself contributes no KL term because it is chosen by the same policy as a function of the observed history. The important point is that the expectation is under \(P_\base^\rho\), so the distribution of \(p_t\) does not depend on the planted cell \(a\).

The one-round feedback carries information through two channels. The first is the local win-loss observation at the played bid. The second is the additional losing-side report on the low-shill branch. Revealing the low/high branch can only increase KL, and the branch probability is the same under all instances. Since the high branch carries no additional buyer-side information, for every \(a\) and every bid \(p\),
\[
\mathrm{KL}\!\left(
Q_{0,p}
\,\middle\|\,
Q_{a,p}
\right)
\le
\mathrm{kl}\!\left(
\operatorname{Bern}(F_{\base}(p))
\,\middle\|\,
\operatorname{Bern}(F_a(p))
\right)
+
\gamma\,
\mathrm{KL}\!\left(
\Lambda^{\mathrm{low}}_{0,p}
\,\middle\|\,
\Lambda^{\mathrm{low}}_{a,p}
\right),
\]
where \(\Lambda^{\mathrm{low}}_{a,p}\) is the law, under instance \(a\), of the one-round low-shill feedback at bid \(p\), including the null symbol \(\bot\) when no losing-side report is observed. We now bound these two terms.

\begin{itemize}
    \item  For the local term, a direct win-loss observation at bid \(p\) is Bernoulli with mean \(F_a(p)\). Since \(F_{\base}(p)\in[3/8,1/2]\) on \(J\), and since \(F_a(p)\in[3/8,3/4]\) under the choice \(\varepsilon\le1/8\), the standard chi-square comparison for Bernoulli laws gives
    \[
    \mathrm{kl}\!\left(
    \operatorname{Bern}(F_{\base}(p))
    \,\middle\|\,
    \operatorname{Bern}(F_a(p))
    \right)
    \le
    \frac{16}{3}\bigl(F_a(p)-F_{\base}(p)\bigr)^2.
    \]
    The perturbation \(F_a-F_{\base}\) is supported only on \(I_a\), and its size is at most \(2\varepsilon\). For every fixed bid \(p\in J\), at most one planted cell contains \(p\). Since \(h\le |J|/4\), we have \(1/N\le 2h/|J|=12h\). Therefore, for every fixed \(p\),
    \[
    \frac1N\sum_{a=1}^N
    \mathrm{kl}\!\left(
    \operatorname{Bern}(F_{\base}(p))
    \,\middle\|\,
    \operatorname{Bern}(F_a(p))
    \right)
    \le
    \frac{16}{3}\cdot(2\varepsilon)^2\cdot 12h
    =
    256h\varepsilon^2.
    \]
    Averaging over the baseline history and summing over \(T\) rounds gives the local contribution
    \[
    256T h\varepsilon^2.
    \]

    \item We now bound the low-shill contribution. Fix a bid \(p\), and let \(S^{\mathrm{low}}\) be distributed as the shill bid conditional on the low branch. For each planted cell \(a\), let \(B_a\sim F_a\), and let \(B_\base\sim F_{\base}\). Define
    \[
    Y^{\mathrm{low}}_{a,p}
    :=
    \begin{cases}
    \bot, & B_a\le p,\\
    \max\{B_a,S^{\mathrm{low}}\}, & B_a>p.
    \end{cases}
    \]
    Let \(\Lambda^{\mathrm{low}}_{a,p}\) be the law of \(Y^{\mathrm{low}}_{a,p}\), and define \(\Lambda^{\mathrm{low}}_{0,p}\) analogously with \(B_\base\sim F_{\base}\). Since \(S^{\mathrm{low}}\) has the same law under \(F_a\) and \(F_{\base}\), and since the map from \((B,S^{\mathrm{low}})\) to \(Y^{\mathrm{low}}_{a,p}\) is common to all instances, data processing gives
    \[
    \mathrm{KL}\!\left(
    \Lambda^{\mathrm{low}}_{0,p}
    \,\middle\|\,
    \Lambda^{\mathrm{low}}_{a,p}
    \right)
    \le
    \mathrm{KL}\!\left(F_{\base}\,\middle\|\,F_a\right).
    \]
    It remains to control the buyer-side divergence. Let \(f_a\) and \(F_{\base}\) be the densities of the continuous parts of \(F_a\) and \(F_{\base}\). The two distributions agree outside \(I_a\), up to the common atom at \(1\). On \(I_a\),
    \[
    f_a(x)-F_{\base}(x)
    =
    \varepsilon
    \left(
    \frac{b_a'(x)}{1-x}
    +
    \frac{b_a(x)}{(1-x)^2}
    \right),
    \]
    and hence, since \(x\in J\) and \(h\le1\),
    \[
    |f_a(x)-F_{\base}(x)|
    \le
    2(C_b+2)\frac{\varepsilon}{h}.
    \]
    Moreover, under \(\varepsilon\le 9h/(64(C_b+2))\), the previous monotonicity calculation gives \(f_a(x)\ge F_{\base}(x)/2\ge 9/32\) on \(J\). Since the perturbation vanishes at the endpoints of \(I_a\), \(\int_{I_a}(f_a-F_{\base})=0\). Using \(\log u\le u-1\), we obtain
    \[
    \begin{aligned}
    \mathrm{KL}(F_{\base}\|F_a)
    &=
    \int_{I_a} F_{\base}(x)\log\frac{F_{\base}(x)}{f_a(x)}\,dx \\
    &\le
    \int_{I_a} F_{\base}(x)\left(\frac{F_{\base}(x)}{f_a(x)}-1\right)\,dx \\
    &=
    -\int_{I_a}(f_a(x)-F_{\base}(x))\,dx
    +
    \int_{I_a}
    \frac{(f_a(x)-F_{\base}(x))^2}{f_a(x)}
    \,dx \\
    &=
    \int_{I_a}
    \frac{(f_a(x)-F_{\base}(x))^2}{f_a(x)}
    \,dx \\
    &\le
    \frac{128(C_b+2)^2}{9}\frac{\varepsilon^2}{h}.
    \end{aligned}
    \]
    Therefore the low-shill contribution over \(T\) rounds is at most
    \[
    \frac{128(C_b+2)^2}{9}T\gamma\frac{\varepsilon^2}{h}.
    \]
\end{itemize}
Combining the two information channels,
\[
\frac1N\sum_{a=1}^N
\mathrm{KL}\!\left(P_\base^\rho\,\middle\|\,P_a^\rho\right)
\le
K T\left(
h\varepsilon^2+\gamma\frac{\varepsilon^2}{h}
\right),
\qquad
K:=256+\frac{128(C_b+2)^2}{9}.
\]

We now convert this reference-KL bound into a testing statement. Let \(\phi(Z)\in[N]\) be any estimator of the planted cell, and define \(E_a:=\{\phi(Z)=a\}\). Since the events \(E_a\) are disjoint and exhaustive under the baseline law,
\[
\frac1N\sum_{a=1}^N P_\base^\rho(E_a)=\frac1N.
\]
By Pinsker's inequality, for every \(a\),
\[
P_a^\rho(E_a)
\le
P_\base^\rho(E_a)
+
\sqrt{\frac12
\mathrm{KL}\!\left(P_\base^\rho\,\middle\|\,P_a^\rho\right)}.
\]
Averaging over \(a\) and applying Jensen's inequality gives
\[
\frac1N\sum_{a=1}^N P_a^\rho(E_a)
\le
\frac1N
+
\sqrt{
\frac{1}{2N}
\sum_{a=1}^N
\mathrm{KL}\!\left(P_\base^\rho\,\middle\|\,P_a^\rho\right)
}.
\]
Thus, if
\[
KT\left(
h\varepsilon^2+\gamma\frac{\varepsilon^2}{h}
\right)
\le
\frac18
\qquad\text{and}\qquad
N\ge4,
\]
then
\[
\frac1N\sum_{a=1}^N P_a^\rho(\phi(Z)=a)
\le
\frac14+\frac14
=
\frac12.
\]
Equivalently,
\[
\frac1N\sum_{a=1}^N P_a^\rho(\phi(Z)\ne a)
\ge
\frac12.
\]

We now connect identification to regret. For each cell \(b\), define
\[
M_b(Z)
:=
\sum_{t=1}^T \mathbf 1\{p_t\in G_b\},
\]
the number of times the learner bids in the good region of cell \(b\). Define the best possible estimator
\[
\phi_\rho(Z)\in \argmax_{b\in[N]}M_b(Z),
\]
breaking ties arbitrarily. Under instance \(a\), every bid outside \(G_a\) loses at least \(\varepsilon/2\). Hence
\[
R_T
\ge
\frac{\varepsilon}{2}
\sum_{t=1}^T \mathbf 1\{p_t\notin G_a\}
=
\frac{\varepsilon}{2}\bigl(T-M_a(Z)\bigr).
\]
If \(\phi_\rho(Z)\ne a\), then some other good region receives at least as many bids as \(G_a\). Since the good regions are disjoint,
\[
M_a(Z)\le T-M_a(Z),
\]
and therefore \(T-M_a(Z)\ge T/2\). Thus
\[
R_T
\ge
\frac{\varepsilon T}{4}
\mathbf 1\{\phi_\rho(Z)\ne a\}.
\]
Taking expectation under \(P_a^\rho\), then averaging over \(a\), gives
\[
\frac1N\sum_{a=1}^N
\mathbb E_a^\rho[R_T]
\ge
\frac{\varepsilon T}{4}
\cdot
\frac1N\sum_{a=1}^N
P_a^\rho(\phi_\rho(Z)\ne a)
\ge
\frac{\varepsilon T}{8}.
\]
Therefore there exists a planted cell \(a\) such that \(\mathbb E_a^\rho[R_T]\ge \varepsilon T/8\), and since the policy \(\rho\) was arbitrary,
\[
\inf_\rho\sup_a \mathbb E_a^\rho[R_T]
\ge
\frac{\varepsilon T}{8}.
\]

It remains to choose \(h\) and \(\varepsilon\). Define
\[
c_\varepsilon
:=
\min\left\{
\frac18,\,
\frac{9}{1536\sqrt{24}(C_b+2)},\,
\frac{1}{\sqrt{200K}}
\right\}.
\]
This constant is fixed once and for all.

First suppose \(\gamma\le T^{-2/3}\). Choose
\[
h=\frac{1}{24}T^{-1/3},
\qquad
\varepsilon=c_\varepsilon T^{-1/3}.
\]
Then \(h\le1/24\), so \(N\ge4\). Also \(\varepsilon/h=24c_\varepsilon\le 9/(64(C_b+2))\), so the CDF monotonicity condition holds. Moreover,
\[
KT\left(
h\varepsilon^2
+
\gamma\frac{\varepsilon^2}{h}
\right)
\le
Kc_\varepsilon^2\left(\frac1{24}+24\gamma T^{2/3}\right)
\le
25Kc_\varepsilon^2
\le
\frac18.
\]
Thus the reference-KL testing condition holds, and the regret lower bound is
\[
\frac{\varepsilon T}{8}
=
\frac{c_\varepsilon}{8}T^{2/3}.
\]

Now suppose \(\gamma>T^{-2/3}\). Choose
\[
h=\min\{\sqrt\gamma, |J|/4\}
=
\min\{\sqrt\gamma,1/24\},
\qquad
\varepsilon=c_\varepsilon T^{-1/2}\gamma^{-1/4}.
\]
Then \(h\le1/24\), so \(N\ge4\). If \(h=\sqrt\gamma\), then \(\varepsilon/h\le c_\varepsilon\). If \(h=1/24\), then \(\gamma\ge1/576\), and hence \(\varepsilon/h\le24\sqrt{24}\,c_\varepsilon\). In both cases, the definition of \(c_\varepsilon\) ensures \(\varepsilon/h\le9/(64(C_b+2))\), so the CDF monotonicity condition holds. Also,
\[
T\left(
h\varepsilon^2
+
\gamma\frac{\varepsilon^2}{h}
\right)
=
c_\varepsilon^2\left(
\frac{h}{\sqrt\gamma}
+
\frac{\sqrt\gamma}{h}
\right).
\]
If \(h=\sqrt\gamma\), the term in parentheses is \(2\). If \(h=1/24\), then \(h/\sqrt\gamma\le1\) and \(\sqrt\gamma/h\le24\), so the term in parentheses is at most \(25\). Hence
\[
KT\left(
h\varepsilon^2
+
\gamma\frac{\varepsilon^2}{h}
\right)
\le
25Kc_\varepsilon^2
\le
\frac18.
\]
Thus the reference-KL testing condition holds, and the regret lower bound is
\[
\frac{\varepsilon T}{8}
=
\frac{c_\varepsilon}{8}\sqrt T\,\gamma^{-1/4}.
\]
Combining the two regimes gives
\[
\inf_\rho\sup_\nu \mathbb E_{\nu,\rho} [R_T]
\ge
\frac{c_\varepsilon}{8}\min\left\{
T^{2/3},
\sqrt T\,\gamma^{-1/4}
\right\}. \qedhere
\]
\end{proof}

\end{document}